\documentclass{article}
\usepackage{graphicx}
\usepackage{fullpage}
\input{preamble.sty}

\title{Provably Extracting the Features from a General Superposition}
\author{Allen Liu \footnote{This work was partially supported by a Miller Research Fellowship while at UC Berkeley}  \\
NYU Courant \\
\texttt{liu.a@nyu.edu}}
\date{}

\begin{document}

\maketitle

\begin{abstract}
It is widely believed that complex machine learning models generally encode features through linear representations. This is the foundational hypothesis behind a vast body of work on interpretability.  A key challenge toward extracting interpretable features, however, is that they exist in superposition. In this work, we study the question of extracting features in superposition from a learning theoretic perspective. We start with the following fundamental setting: we are given query access to a function 
\[
f(x)=\sum_{i=1}^n \sigma_i(v_i^\top x),
\]
where each unit vector $v_i$ encodes a feature direction and $\sigma_i:\R\to\R$ is an arbitrary response function and our goal is to recover the $v_i$ and the function $f$.

In learning-theoretic terms, superposition refers to the \emph{overcomplete regime}, when the number of features is larger than the underlying dimension (i.e. $n > d$), which has proven especially challenging for typical algorithmic approaches. Our main result is an efficient query algorithm that, from noisy oracle access to $f$, identifies all feature directions whose responses are non-degenerate and reconstructs the function $f$.  Crucially, our algorithm works in a significantly more general setting than all related prior results. We allow for essentially arbitrary superpositions, only requiring that $v_i, v_j$ are not nearly identical for $i \neq j$, and allowing for general response functions $\sigma_i$. At a high level, our algorithm introduces an approach for searching in Fourier space by iteratively refining the search space to locate the hidden directions $v_i$.
\end{abstract}

\newpage

\section{Introduction}

While modern machine learning models are incredibly complex, a foundational viewpoint in our quest to understand them is the notion of \emph{linear representations} \cite{alain2016understanding,elhage2022toy, park2023linear, olah_linear_representation_2024,golowich2025sequenceslogitsreveallow}. The hypothesis is that salient features correspond to directions $v\in\R^d$ in some representation space, and the response to a feature depends only on the one–dimensional projection $v^\top x$ via some \emph{ridge function} say $f(x) = \sigma(v^\top x)$. This hypothesis has served as the basis for a vast body of work on extracting features, activation steering, and interpretability more broadly \cite{kim2018interpretability, turner2023steering, olah_linear_representation_2024}. 

From a theoretical perspective, even the study of simple, single-feature functions, that depend only on the projection of the input onto a single direction, has led to a rich body of work in computational learning theory through generalized linear models \cite{mccullagh2019generalized, kakade2011efficient, chen2020classification}, single-index models \cite{ichimura1993semiparametric, dudeja2018learning, bietti2022learning, gollakota2023agnostically, damian2023smoothing, damian2024computational, zarifis2024robustly}, and also related problems such as non-Gaussian component analysis \cite{blanchard2006search, diakonikolas2022non}.  However, further adding to the challenge is the fact that for most models or functions that we want to study, the output depends on many features of the input. We formalize this in the most basic setting, when the output is a sum of the responses to individual features i.e. a \emph{sum of ridge functions}
\[
f(x)\;=\;\sum_{i=1}^n \sigma_i(v_i^\top x),
\]
where each $v_i\in\mathbb R^d$ is a unit ``feature direction'' and $\sigma_i:\mathbb R\to\mathbb R$ is an arbitrary univariate response. 

A key concept in feature learning and interpretability is the notion of \emph{superposition}, formalized in \cite{elhage2022toy}. Superposition refers to the fact that the number of features is generally much larger than the dimension of the representation, and has been highlighted as an important obstacle and potentially fundamental barrier to extracting interpretable features  \cite{elhage2022toy,olah_linear_representation_2024, cunningham2023sparse}. As stated in \cite{elhage2022solu}, ``superposition means there is \emph{no basis} in which activations are interpretable". This means that we cannot obtain an ideal decomposition into features that are disentangled and correspond to human-interpretable concepts. 

In practice, sparse autoencoders and related dictionary-learning methods have been used to recover interpretable features from representations in superposition in some cases \cite{bricken2023monosemanticity,cunningham2023sparse}. However, these empirical methods are optimization heuristics and can exhibit instability, incompleteness, and overfitting \cite{olah_linear_representation_2024}, while theoretical guarantees apply only under strong structural assumptions on the underlying features. This motivates us to ask the following question from a learning-theoretic perspective:

\begin{gquestion}
Can we study superposition and the obstacles it poses from the perspective of computational learning theory? Can we formalize the challenges that arise in simple settings and develop algorithms to address them?   
\end{gquestion}

In learning-theoretic terms, superposition corresponds to the \emph{overcomplete} regime where the number of features $n$ exceeds the ambient dimension $d$. Yet, this overcomplete regime has also proven particularly challenging from an algorithmic perspective, posing a barrier for common approaches such as moment and tensor methods \cite{hillar2013most, anandkumar2013overcomplete, anandkumar2014tensor}.

To formalize the problem setup, it remains to specify how we can access the function $f$. From the perspective of extracting features from a trained model, the natural question is to recover the features given the ability to query the function. In fact, this is the typical setting for interpretability analyses, where one takes a trained model and then analyzes its activations and tries to understand the learned representations by plugging in carefully chosen inputs.  Thus, we study the following query learning setting: we assume black-box \emph{query access} to some function $f_{\sim}$ with $\|f_{\sim} - f\|_\infty\le \varepsilon$.  We now ask:

\begin{aquestion}
Given queries to a function that is (close to) a sum of features, can we learn the underlying features?  Are there barriers to learning if the features exhibit superposition?  
\end{aquestion}

Beyond being natural from the perspective of feature learning and interpretability, this query learning model also allows us to sidestep the difficulty of positing realistic distributional assumptions. Typical learning setups often assume $x$ is drawn from some distribution, such as a Gaussian, and then ask to recover $f$ given polynomially many samples $(x, f(x))$.  However, it is difficult to posit tractable but realistic distributional assumptions \----  common assumptions like $x$ being Gaussian are unrealistic \---- and furthermore, even then there are still computational barriers in simple settings with just a single feature \cite{goel2020statistical, song2021cryptographic, damian2024computational}. The query model also captures natural settings for model distillation and stealing, where a learner tries to recover something about a trained machine learning model through black-box API access \cite{tramer2016stealing, finlayson2024logits, carlini2024stealing}.

In light of the discussion above, the main question we pose is important both for understanding how to extract interpretable features \footnote{Also see e.g. \cite{agarwal} which proposes Neural Additive Models (NAMs), a special case of our model, as an architecture for interpretable deep learning} and from a fundamental computational learning theory perspective, where it builds on a vast existing body of work on learning single hidden layer neural networks. Within this literature there are two main points of comparison.  First, there are numerous results in the standard ``passive" setting, but the bottom line is that here, there is substantial evidence of computational hardness.  There are statistical query lower bounds even for settings that are much simpler, special cases of our setting such as learning a single ridge function with general activation \cite{goel2020statistical} and learning a linear combination of many ridge functions with fixed e.g. ReLU activation \cite{diakonikolas2020algorithms, goel2020superpolynomial}, suggesting that a general learning guarantee is out of reach for efficient algorithms. Second, in query learning models, there has been work specifically in the case of ReLU neural networks i.e. $\sigma_i(z) = \text{ReLU}(z)$ for all $i \in [n]$, but these algorithms are extremely tailored to specifically the ReLU activation function, relying on trying to find the kink in the ReLU \cite{rolnick2020reverse, chen2021efficientlylearninghiddenlayer, daniely2021exact}. From the lens of feature learning, we expect that response functions can be far more general than just ReLU, and thus previous work falls significantly short of a clean resolution.

\paragraph{Main Result} Our main result is an algorithm for learning a general sum of features under mild and information-theoretically necessary assumptions \---- nontrivial separation between the feature directions $v_i$ and Lipschitzness of the functions $\sigma_i(\cdot)$ (see \Cref{assume:nondegen}). Our algorithm has query complexity and runtime polynomial in all relevant parameters and achieves the following guarantees: 
\begin{enumerate}
  \item \emph{Identifies} all nonlinear feature directions $\{v_i\}$ up to sign and $\eps$-error
  \item \emph{Reconstructs} the associated univariate responses on $[-R,R]$, yielding a uniform $\eps$-approximation to $f$ on the domain $\|x\|\le R$.
\end{enumerate}

In other words, if some trained model is close to a linear combination of features in some known representation space, then we can efficiently recover the model from queries. The formal statements are in \Cref{thm:main-learning} and \Cref{coro:identifiability}. Note that if some of the $\sigma_i(\cdot)$ were linear, then these directions would not be identifiable for trivial reasons.

We emphasize that our result holds for overcomplete $v_i$, even when some of the $v_i$ are highly correlated and for general response functions $\sigma_i$.  This circumvents the aforementioned computational hardness in the passive setting and significantly generalizes previous learning results which either require restrictive assumptions on the $\sigma_i$ (e.g. ReLU) \cite{chen2021efficientlylearninghiddenlayer, daniely2021exact} or the $v_i$ (e.g. linear independence or near-orthogonality) \cite{sedghi2016provable, oko2024learning}.

\paragraph{Note on Identifiability} We also highlight the importance of the identifiability aspect of our result from the lens of feature learning and interpretability.  As formulated in \cite{elhage2022toy}, one of the fundamental challenges in extracting features from modern models, and interpretability more broadly, is \emph{superposition} \---- when the number of features is larger than the ambient dimension, it makes the features nonidentifiable.  Indeed, this would be the case if the response functions were linear.  However, from this viewpoint, our results actually give a new reason for optimism. A key takeaway from our results is:
\begin{takeaway} For a function that is a sum of many features, the individual features are identifiable, even in a superposition, as long as the response functions are nonlinear.    
\end{takeaway}

\paragraph{High Level Approach} Key to our result is a different type of algorithmic approach that departs from the typical method-of-moments recipe that is ubiquitous for learning latent variable models, especially over Euclidean spaces.  The starting point is the basic observation that the Fourier transform $\wh{f}$ of $f$ is ``sparse''. Ignoring integrability issues for now, for a single ridge function $f(x)=\sigma(v^\top x)$, the Fourier transform is nonzero only on the line $\{tv \}_{t \in \R}$. Thus when $f$ is a sum of ridge functions, its Fourier transform is supported on only the lines through the origin in the directions $v_i$. This observation was made in \cite{eldan2016power} and has been used for proving negative results and depth separations about the expressibility of shallow neural networks.  However, here we will use this structural property algorithmically for a positive result.

At a high level, to recover the directions, we design an algorithm that can ``locate the mass" in Fourier space.   This paradigm has already been successful in boolean \cite{goldreich1989hard, kushilevitz1991learning} and discrete \cite{hassanieh2012nearly} settings. Our algorithm draws inspiration from this approach \---- the main subroutine (Algorithm~\ref{alg:search}) involves an iterative algorithm for searching in Fourier space to locate the hidden directions.  However, compared to discrete spaces, the continuity and unboundedness of Euclidean space pose technical challenges \----  we emphasize the conceptual novelty here, that through carefully chosen geometric constructions and filter functions, we develop an algorithm that efficiently searches over a high-dimensional Euclidean Fourier domain. We provide a more detailed overview of our techniques in \Cref{sec:tech-overview}.

\subsection{Related Work}

\paragraph{GLMs and single and multi‑index models.} There is a substantial body of work on GLM/SIM estimators, which recover a function that depends only on the projection of the input onto one direction, under various distributional and noise assumptions \cite{mccullagh2019generalized,ichimura1993semiparametric, kakade2011efficient, dudeja2018learning,bietti2022learning,damian2023smoothing,gollakota2023agnostically,zarifis2024robustly,damian2024computational}.  Multi-index models (MIMs) generalize SIMs to functions that depend only on the projection of the input onto a constant dimensional subspace. Recently there has also been a growing body of work on learning MIMs \cite{bietti2023learning, diakonikolas2024agnostically, mousavi2024learning}.  We refer the reader to \cite{bruna2025survey} for a more detailed overview of this literature.  Note that our setting, involving a linear combination of ridge functions, is very different because the function actually depends on the full $d$-dimensional space.

\paragraph{Shallow Neural Networks and Generalizations.} In our setting, when all of the $\sigma_i$ are some fixed activation function such as sigmoid or ReLU, then $f$ can be viewed as a two-layer neural network.  There is a large body of work on learning two-layer neural networks from samples, which aims to characterize the boundaries of efficient algorithms and computational hardness e.g. \cite{awasthi2021efficient, diakonikolas2020algorithms, ge2018learning, goel2020superpolynomial, chen24faster}. Beyond understanding the computational landscape, there is also a line of work towards understanding specifically the dynamics of gradient descent for training shallow networks e.g. \cite{janzamin2016beatingperilsnonconvexityguaranteed, ge2017learningonehiddenlayerneuralnetworks, zhong2017recoveryguaranteesonehiddenlayerneural}. However in the passive setting, there are exponential in $\min(d,n)$ SQ lower bounds, suggesting that a general, efficient (i.e. $\poly(d,n)$) time algorithm is not possible.

There has also been work on learning neural networks with query access \cite{rolnick2020reverse, chen2021efficientlylearninghiddenlayer, daniely2021exact}. However, as mentioned earlier, these results are extremely tailored to ReLU activations, whereas our result allows for arbitrary activations $\sigma_i$.

Linear combinations of more general activation functions are considered in \cite{sedghi2016provable,oko2024learning}. These works also study a passive learning setting, and thus require structural assumptions on the $v_i$, namely that they are linearly independent in the former, and nearly orthogonal in the latter, whereas we do not need any such conditions.

\paragraph{Fourier search in discrete spaces.} Our algorithm of searching in a high-dimensional Fourier space draws inspiration from the Goldreich–Levin heavy‑Fourier‑coefficient search algorithm \cite{goldreich1989hard,kushilevitz1991learning} and also bears resemblance to algorithms used for computing sparse Fourier transforms \cite{hassanieh2012nearly}.  However, a key difference is that we give an efficient algorithm for searching over a high-dimensional Euclidean space. We remark that the connection between Fourier sparsity and expressing functions as a sum of ridge functions dates back to the classical works \cite{barron, candes, donoho} but our focus here is on making this connection algorithmic. 

\paragraph{Query learning and model stealing.} With the rise of APIs that provide users with black-box access to trained machine learning models, there has been increased interest in understanding what we can extract from such an API, both from a theoretical \cite{mahajan2023learning, liu2025model, golowich2025sequenceslogitsreveallow,golowich2025provablylearningmodernlanguage} and practical perspective \cite{tramer2016stealing, finlayson2024logits, carlini2024stealing}. Our result says that if a model is close to a sum of features in some representation space that we know, then we can recover the model from queries.

\section{Preliminaries}

We begin with some basic notation and facts that will be used throughout the paper.  Throughout, we use the convention $\ii = \sqrt{-1}$.

\subsection{Ridge Functions (Features)}

A ridge function is a function that depends only on the projection of the input onto a single direction.  

\begin{definition}[Ridge Function (Feature)]\label{def:ridge}
A function $f: \R^d \rightarrow \R$ is a ridge function if it can be written as $f(x) = \sigma(v^\top x)$ for some unit vector $v \in \R^d$ and univariate function $\sigma$. 
\end{definition}
\begin{remark}
Whenever we write a ridge function in the above form, we will assume that $\norm{v} = 1$ (we can always ensure this by just rescaling $\sigma$).
\end{remark}

Given a ridge function, we can think of $v$ as encoding a feature direction, and the function $\sigma$ as an activation that controls how strongly the output responds to the strength of the feature. More generally, we can consider a sum or linear combination of features, each involving a different direction $v_i$ and possibly different activation $\sigma_i$.

\begin{definition}[Sum of Features]\label{def:lin-comb}
A function $f: \R^d \rightarrow \R$ is a sum of $n$ features if it can be written as a linear combination of $n$ ridge functions i.e.
\[
f(x) =  \sigma_1(v_1^\top x) + \dots + \sigma_n(v_n^\top x) \,.
\]
\end{definition}
Note that the terms sum and linear combination can be used interchangeably since we can always absorb the coefficients of a linear combination into the functions $\sigma_i$ themselves.

\subsection{Learning Setup}\label{sec:setup}

The main goal in this paper will be to recover an unknown sum of features from queries.  We now describe the details of the problem setup.  There is an unknown function $f:\R^d \rightarrow \R$ with
\[
f(x) = \sigma_1(v_1^\top x) + \dots + \sigma_n(v_n^\top x) \,.
\]
We receive query access to a function $f_{\sim}$ that satisfies $\norm{f - f_{\sim}}_{\infty} \leq \eps$ and our goal will be to recover a description of a function that is close to $f$.  We will be mostly interested in the overcomplete regime where $n \geq d$ \---- this regime in particular has proven challenging for designing efficient algorithms \footnote{When $n < d$, we can first learn the subspace spanned by the $v_i$ via a standard tensor method and then we can project onto this subspace to reduce to the case where $n \geq d$.}.
\\\\
We will make a few assumptions on the $v_i, \sigma_i$ to rule out degenerate or pathological examples.
\begin{assumption}\label{assume:nondegen}
We make the following assumptions on $f$:
\begin{itemize}
\item The functions $\sigma_i$ all satisfy $\sigma_i(0) = 0$
\item The functions $\sigma_i$ are all $L$-Lipschitz
\item The sine of the angle between $v_i$ and $v_j$ is at least $\gamma$ whenever $i \neq j$  
\end{itemize}
\end{assumption}
\begin{remark}
The first assumption is just for the sake of normalization, and the second one is without loss of generality since we can always just query $f_{\sim}(0)$ and subtract it off. The Lipschitzness of the activations $\sigma_i$ is necessary since otherwise, even for a single function in one dimension, it would be impossible to approximate it from queries.  While the final assumption may not be strictly necessary for just recovering the function, it is necessary for being able to recover the individual directions $v_i$ as it rules out degenerate examples where features are (almost) identical and cancel each other out.      
\end{remark}

We will also mention the following assumption, that the $\sigma_i$ are bounded.  While this assumption is not necessary, it simplifies the exposition.  We will first give a learning algorithm under this assumption, and then show in \Cref{sec:remove-bounded} (see \Cref{thm:main-learning}) how to remove this secondary assumption.  

\begin{assumption}\label{assume:bounded}
Assume that the functions $\sigma_i$ satisfy $|\sigma_i(x)| \leq 1$ for all $x \in \R, i \in [n]$.
\end{assumption}

The majority of the paper will be devoted to proving the following theorem:
\begin{theorem}\label{thm:weaker}
Under \Cref{assume:nondegen} and \Cref{assume:bounded}, for any target accuracy $\eps' $, target domain $R$, and failure probability $\delta$, there is some $N = \poly(d, L, R, n,  1/\gamma, 1/\eps', \log 1/\delta)$ such that if $\eps < 1/N$, there is an algorithm that makes $\poly(N)$ queries and runs in $\poly(N)$ time and outputs a sum of features 
\[
\wt{f}(x) =  \sigma_1'({v_1'}^\top x) + \dots +  \sigma_n'({v_n'}^\top x)
\]
such that with probability  $1 - \delta$, for all $x$ with $\norm{x} \leq R$, $|f(x) - \wt{f}(x)| \leq  \eps'$.
\end{theorem}

\begin{remark}
Note that with finitely many queries, it is only possible to guarantee closeness over a bounded domain, as opposed to everywhere, even for a single function in one dimension, and thus the restriction to $\norm{x} \leq R$ in the above theorem is necessary. 
\end{remark}

\Cref{thm:weaker} in fact guarantees not just recovering the function, but also recovering all of the hidden directions $v_i$ for which the function $\sigma_i$ is nontrivial (see \Cref{lem:find-directions} for a formal statement).  If some $\sigma_i$ were constant, then of course, recovering the associated direction is impossible.

\subsection{Main Results}

The final theorem and corresponding identifiability result are stated below. After completing the proof of \Cref{thm:weaker}, we complete the proof of the stronger version through a direct reduction in \Cref{sec:remove-bounded}.

\begin{theorem}\label{thm:main-learning}
For any target accuracy $\eps' $, target domain $R$, and failure probability $\delta$, there is some $N = \poly(d, L, R, n,  1/\gamma, 1/\eps', \log 1/\delta)$ such that if $\eps < 1/N$, then under \Cref{assume:nondegen}, there is an algorithm that makes $\poly(N)$ queries and $\poly(N)$ runtime and outputs a sum of features 
\[
\wt{f}(x) =  \sigma_1'({v_1'}^\top x) + \dots +  \sigma_n'({v_n'}^\top x)
\]
such that with probability  $1 - \delta$, for all $x$ with $\norm{x} \leq R$, $|f(x) - \wt{f}(x)| \leq  \eps'$.
\end{theorem}

\begin{corollary}\label{coro:identifiability}
In the same setting as \Cref{thm:main-learning}, the algorithm can also guarantee to return directions $v_1', \dots , v_s'$ for some $s \leq n$ such that  
\begin{itemize}
\item Each returned direction $v_j'$ satisfies $\min(\norm{v_j' - v_i}, \norm{v_j' + v_i}) \leq \eps'$ for some $i \in [n]$
\item For each $i \in [n]$ such that the function $\sigma_i(\cdot )$ has $\max_{z \in [-R,R]}|\sigma_i(z) - (az + b)| \geq \eps'$ for any linear function $az + b$, there must be some returned direction $v_j'$ with $\min(\norm{v_j' - v_i}, \norm{v_j' + v_i}) \leq \eps'$
\end{itemize}
\end{corollary}

\section{Technical Overview}\label{sec:tech-overview}

In this section, we give an overview of the techniques and key ingredients that go into the proof of \Cref{thm:weaker}.

Recall that the starting point is the intuition that if $f$ is a sum of $n$ features in directions $v_1, \dots , v_n$, then the Fourier transform $\wh{f}$ is nonzero only on the lines $\{tv_i \}_{t \in \R}$.  First, to make this formal, we need to resolve the integrability issues.  Recall the Fourier transform is
defined as
\[
\wh{f}(y) = \frac{1}{(2\pi)^{d/2}}\int_{\R^d} e^{-\ii y^\top x} f(x) dx 
\]
but for the above to be well-defined, we require $f(x)$ to be integrable over $\R^d$. To address this, we will multiply by a Gaussian reweighting, which will ensure integrability. We define the Gaussian-reweighted function
\[
  f^{(\ell)}(x) \;=\; f(x)\,\exp\!\big(-\|x\|^2/(2\ell^2)\big),
\]
and we will choose $\ell$ sufficiently large. With this definition, it is clear that the Fourier transform $\wh{f^{(\ell)}}$ is well-defined. Furthermore, it remains concentrated around the lines $\{tv_i \}_{t \in \R}$. To see why, for a single ridge function $\rho(x) = \sigma(v^\top x)$, the Fourier transform can be written as (see \Cref{claim:fourier-formula})
\[
  \widehat{\rho^{(\ell)}}(y)
  \;=\;
  \widehat{\sigma^{(\ell)}}(v^\top y)\cdot
  \ell^{\,d-1}\exp\!\Big(
    -\tfrac{\ell^2}{2}\,\|y-(v^\top y)v\|^2
  \Big)  \,,
\]
and the exponential term implies that it concentrates on a \emph{tube} of width $\approx 1/\ell$ around the line
$\{t v:t\in\R\}$. A sum of ridge functions thus yields a sum of such tubes in Fourier space.

\subsection{High Level Idea}

To recover the directions $v_i$, we want an algorithm that can ``locate the mass" in Fourier space. The famous result of Goldreich and Levin gives an efficient algorithm that searches for heavy Fourier coefficients of a boolean function \---- even though there are exponentially many possible coefficients, the algorithm gets around this by querying the total weight on various slices of the hypercube, and only zooming in further on the slices that have nontrivial weight.  Our algorithm draws inspiration from this approach but we must overcome additional challenges in Euclidean space.

To iteratively refine the search space, we use hyperplane-like slices.  Given an orthonormal basis, say $b_1, \dots , b_d$, roughly we first search over hyperplanes orthogonal to $b_1$ (say discretized to some grid).  Then among the hyperplanes that have nontrivial Fourier weight, we search within those along hyperplanes orthogonal to both $b_1$ and $b_2$.  Iterating this procedure, at each level $k$, we maintain a collection of ``candidate" $k$-tuples, say $\{ (\alpha_1^{(i)}, \dots , \alpha_k^{(i)}) \}_{i \in [m]}$, for the first $k$ coordinates.  Then within each of these, we enumerate the possibilities for the next coordinate $\alpha_{k+1}$ and recurse on all candidates with nontrivial Fourier mass.

This high-level approach is the foundation of our algorithm. To actually implement this approach, we need to address two key aspects, which we discuss in the sections below:
\begin{enumerate}
\item How can we estimate the total Fourier mass on a hyperplane?
\item How can we guarantee that the algorithm has bounded recursion while ensuring that we recover all relevant directions?
\end{enumerate}

\subsection{Estimating the Fourier Mass}

Again, to ensure that the notion of ``total weight on a hyperplane" is well-defined, we need to address integrability. Rather than restricting strictly to a hyperplane, we will add some small Gaussian thickening. Formally, for any function $g$, vector $v \in \R^d$ and $A \succeq 0$, we define  
\begin{equation}\label{eq:intro-mass-def}
I^\star_{g}(v,A) \defeq \int_{\R^d} \left|\wh{g}(y)\right|^2 e^{-(y-v)^\top A (y-v)}\,dy \,.
\end{equation}
To simulate a hyperplane orthogonal to say $b_1, \dots , b_k$, we set $A = C(b_1b_1^\top + \dots  + b_kb_k^\top)$ for some sufficiently large $C$ (for technical reasons later on, we will need different values of $C$ for different coordinates to bound the branching in the search algorithm). This choice of $A$ ensures that the integral is dominated by the Fourier mass on points that are close to the hyperplane through $v$, orthogonal to $b_1, \dots , b_k$.

Now our goal will be to estimate $I^\star_{f^{(\ell)}}(v,A) $ for various $v,A$.  A simple, but critical observation, is that by substituting the definition 
\[
\wh{g}(y) = \frac{1}{(2\pi)^{d/2}} \int_{\R^d} e^{-\ii y^\top x} g(x) dx
\]
into the expression for $I^\star_g(v,A)$ and switching the order of integration and explicitly computing the Gaussian integrals, we can rewrite it as an expectation over \emph{two-point correlations} of $g$ (see \Cref{claim:gaussian-reweight}):
\[
  I^\star_g(v,A)
  \;=\;
  \E_{\Delta\sim\mathcal N(0,2A)}\!\Big[
    e^{-i v^\top \Delta}\cdot
    \int_{\R^d} g(x)\,g(x+\Delta)\,dx
  \Big] \,.
\]
Now setting $g(x) = f^{(\ell)}(x) = f(x) e^{-\norm{x}^2/(2\ell^2)}$, we can write the above as an expectation over both $x$ and $\Delta$ involving values of $f$ itself:
\begin{equation}\label{eq:intro-mass-two-point}
I^\star_{f^{(\ell)}}(v,A)
=
(\pi \ell^2)^{\frac d2}
\E_{x\sim\mathcal N(0,\frac{\ell^2}{2} I_d), \Delta\sim\mathcal N(0,2A)}\!\Big[
    e^{-\frac{\|\Delta\|^2}{4\ell^2} - i v^\top \Delta}\cdot
    f\!\left(x-\tfrac12\Delta\right)\,
    f\!\left(x+\tfrac12\Delta\right)
  \Big] \,.
\end{equation}

Since $f(x)$ is bounded, the quantity inside the expectation is bounded, and we can estimate the expectation by sampling and querying the values $f\left(x-\tfrac12\Delta\right), f\left(x+\tfrac12\Delta\right)$ for many $x, \Delta$ (see Algorithm~\ref{alg:estimate-Fourier}). This is the key subroutine that allows us to estimate the Fourier mass on various hyperplane-like slices. Stated more concisely:

\begin{corollary}[Informal, see \Cref{coro:concentration}]
We can estimate $I^\star_{f^{(\ell)}}(v,A)$ to within $\eps \cdot (\pi \ell^2)^{\frac d2}$ additive error with high probability using polynomially many queries.
\end{corollary}

\begin{remark}\label{rem:reweighting-choice}
While the Gaussian thickening in \eqref{eq:intro-mass-def} is natural, we remark that the choice of the weight function is actually critical.  For instance, if we had used uniform over a box or ball instead of a Gaussian, then the analogous formula could not be easily written as an expectation over a probability distribution (due to lack of positivity).  This would then make it much harder to estimate via sampling due to difficulties in controlling the variance. 
\end{remark}

Note that the scaling $(\pi \ell^2)^{\frac d2}$ is the ``right" scaling. To see this intuitively, if $|f(x)| = 1$ everywhere, then $\int_{\R^d} |\wh{f^{(\ell)}}(y)|^2 dy= \int_{\R^d} |f^{(\ell)}(x)|^2 dx = (\pi \ell^2)^{\frac d2}$. Estimating the inner expectation in \eqref{eq:intro-mass-two-point} to within $\eps$ additive error thus roughly corresponds to an additive error of $\eps$-fraction of the total Fourier mass. We will always aim to preserve at least an inverse polynomial fraction of the Fourier mass when we localize, so we will only ever need to estimate $I^\star_{f^{(\ell)}}(v,A)$ to within inverse polynomial (times $(\pi \ell^2)^{\frac d2}$) accuracy.

\subsection{Bounding the Search Algorithm}\label{sec:integrability-overview}

Recall the Fourier mass of $f^{(\ell)}$ is concentrated on tubes centered around the lines $\{tv_i \}_{t \in \R, i \in [n]}$. First we ensure that the projections of the lines $\{tv_i \}_{t \in \R, i \in [n]}$ onto the plane formed by the first two coordinates $b_1, b_2$ are still separated lines. To achieve this, we can simply choose the orthonormal basis $b_1,\ldots,b_d$ randomly. A direct anti-concentration argument and union bound shows that then, the projected lines are separated with high probability (see \Cref{claim:anticoncentration} for details).

This means that once we fix the first two coordinates $(\alpha_1, \alpha_2)$, as long as $(\alpha_1, \alpha_2)$ is bounded away from the origin, at most one of the lines $tv_i$ can still be close to any point whose first two coordinates are $(\alpha_1, \alpha_2)$.

\paragraph{Bounding the Number of Candidates} We now discuss bounding the total number of candidates that our algorithm recurses on. We will simply brute-force over a sufficiently fine grid for the first two coordinates (with a neighborhood around the origin cut out). The number of candidates for the first two coordinates is then polynomially bounded. We will focus on ensuring that the number of candidates does not blow up as we search over the remaining coordinates.

Consider a candidate for the first $k$ coordinates $\alpha_1,\ldots,\alpha_k$ for $k \geq 2$.  We then probe candidates for $\alpha_{k+1}$ by evaluating $I^\star_{f^{(\ell)}}(v,A)$
at
\[
  v=\sum_{j\le k+1}\alpha_j b_j,
  \qquad
  A=\sum_{j\le k+1}K_j\,b_j b_j^\top,
\]
with $K_1 = K_2 = K_{k+1} = C_2 , K_3=\cdots=K_k = C_1$ with $C_2 \gg C_1$.  To see why we need different scales for different coordinates, ignore the first two coordinates (since we are brute forcing over these).  For the remaining coordinates, when searching over say the $k+1$st coordinate $\alpha_{k+1}$, we search over a finer grid, of width $\approx \frac{1}{\sqrt{C_2}}$.

Initially, we view each value on this grid as a candidate but then we filter down to only the candidates for which the mass $I^\star_{f^{(\ell)}}(v,A)$ is nontrivial where $v = \sum_{j\le k+1}\alpha_j b_j$. At this point, if we kept and recursed on all of the remaining candidates, then this could blow up the number of candidates by a constant factor (since e.g. both of the grid points on either side of the hidden Fourier spike would be valid candidates).  Instead we choose only a sufficiently separated set of candidates. Naively, this runs into the potential issue that we could be losing a constant fraction of the Fourier mass if the true spike is between two candidate grid points. 

The fix is that after probing with a fine grid and then pruning, we relax the width of the Gaussian for all steps afterward. Once we choose a value for the $k+1$st coordinate, we can relax the width of the Gaussian in that direction to $\approx \frac{1}{\sqrt{C_1}}$ to retain all but a negligible fraction of the Fourier mass from the previous iteration of the recursion. Stated more precisely, for 
\[
A_{k} = \sum_{j\le k}K_j\,b_j b_j^\top, A_{k+1} = \sum_{j\le k+1}K_j\,b_j b_j^\top
\] 
where $K_1 = K_2 = C_2 , K_3=\cdots= K_k = K_{k+1} = C_1$, we ensure that $I^\star_{f^{(\ell)}}(v,A_k) - I^\star_{f^{(\ell)}}(v,A_{k+1})$ is sufficiently small. See Algorithm~\ref{alg:search} for the full details of the search algorithm. 

Recall that after fixing the first two coordinates $(\alpha_1, \alpha_2)$, there is at most one line $tv_i$ for $i \in [n]$ that can be close. This essentially uniquely determines the remaining coordinates. Putting everything together, we have:

\begin{claim}[Informal, see \Cref{claim:bounded-recursion}]
Assuming that the projections of the lines $\{tv_i \}_{t \in \R, i \in [n]}$ onto the plane formed by the first two coordinates $b_1, b_2$ are separated, then for each choice for the first two coordinates away from the origin, the algorithm recurses on at most one candidate for each of the remaining coordinates. Thus, the total number of candidates that the algorithm recurses on is polynomially bounded.
\end{claim}

\paragraph{Ensuring Completeness} Finally, we discuss how we guarantee that when we run the above search algorithm, we recover all relevant directions. For each branch, we recurse until we have chosen all $d$ coordinates $(\alpha_1, \dots , \alpha_d)$ (or the branch terminates with no valid candidates). The collection of points that we recover at the end gives us a collection of directions. Just from the fact that the Fourier transform $\wh{f^{(\ell)}}$ is concentrated on tubes around the lines $\{tv_i \}_{t \in \R, i \in [n]}$, we can argue that the recovered directions must all be close to some true direction $v_i$.

We now discuss how to ensure that this collection is complete i.e. recovers all of the directions $v_i$. To recover a direction $v_i$, it suffices to show that there is a non-negligible amount of Fourier mass associated with the line $\{tv_i\}$ at some scale bounded away from $0$.  In fact, we need $t \gg 1/\ell$ since the ``tubes" around the lines on which $\wh{f^{(\ell)}}$ is concentrated now have width $\approx 1/\ell$. The main algorithmic statement is:

\begin{claim}[Informal, see \Cref{claim:find-good-point}]
If for some direction $v_i$, there exists $t \gg 1/\ell$ such that the Fourier mass of $f^{(\ell)}$ around $tv_i$ is at least an inverse polynomial fraction of the total Fourier mass, then the search algorithm will recover a direction close to $v_i$.
\end{claim}

Intuitively, this claim holds because if the Fourier mass around some point $tv_i$ is nontrivial for some $t \gg 1/\ell$, then one of the candidates that we search over will be close to $tv_i$ and this will ensure that we recover the direction $v_i$. We now proceed with this intuitive understanding  (see the proof of \Cref{claim:find-good-point} for a more formal treatment).

The previous condition ends up being essentially equivalent to finding some $t \gg 1/\ell$ such that $\wh{\sigma_i^{(\ell)}}(t)$ is nontrivial. While the needed condition seems intuitive for natural function classes, ensuring it requires additional work.  As an example, if $\sigma_i(\cdot)$ were a polynomial function like $x^3$, then actually this need not be true.  For large $\ell$, we can easily compute that the Fourier transform of $\sigma_i^{(\ell)}(x) = x^3 e^{-\norm{x}^2/(2\ell^2)}$ \emph{decays exponentially with width scale $1/\ell$}. 

How can we rule out such examples?  Because the functions $\sigma_i(\cdot)$ are Lipschitz, we can choose $\ell$ sufficiently large so that the only way $\sigma_i( \cdot )$ can be close to a polynomial is if it is linear.  For $\sigma_i(\cdot)$ that are close to linear, there is an obvious lack of identifiability since we could easily have two different sets of vectors $\{ v_1, \dots , v_s \}$ and $\{v_1' , \dots , v_s' \}$ such that $v_1 + \dots + v_s = v_1' + \dots + v_s'$ which of course implies that for all $x$, 
\[
v_1^\top x + \dots + v_s^\top x = {v_1'}^\top x + \dots +  {v_s'}^\top x \,.
\]
Fortunately, this does not affect our ability to recover the overall function since we can learn the ``linear part" separately.  Crucially, we actually show that, in some sense, \emph{linear response functions are the only obstruction}. For $\sigma_i(\cdot)$ that are not too close to linear, we show, via careful analysis, that there must be some $t \gtrsim 1/\sqrt{\ell} \gg 1/\ell$ such that $\wh{\sigma_i^{(\ell)}}(t)$ is non-negligible (see \Cref{lem:nondegen-fourier}). This can then be used to argue that our algorithm must recover the corresponding direction $v_i$. The identifiability of the directions $v_i$ for which the response function $\sigma_i(\cdot)$ is non-linear then follows as a consequence of this characterization.

\subsection{Function Recovery}

Once we recover the directions $v_i$, we can recover the function by interpolating in Fourier space. We estimate the Fourier transform $\wh{f^{(\ell)}}(y)$ for $y$ on a bounded discrete grid on each of the lines $\{tv_i \}_{t \in \R, i \in [n]}$. Then for each line, we apply the Fourier inversion formula, interpolating between the grid points and zeroing out the part outside the bounded grid. We show that this approximately recovers each of the ridge functions $\sigma_i(v_i^\top x)$, and thus adding up the reconstructions along each of the lines recovers $f$.

The main bound that we need is \Cref{claim:truncated-inversion} which shows that for any bounded Lipschitz function over $\R$, we can reconstruct it to within a small error by interpolating its Fourier transform on a bounded discrete grid. We then apply this bound in \Cref{lem:recover-function} to control the overall error in recovering the function $f$.

\subsection{Organization}

In \Cref{sec:basic-properties}, we present some general bounds on functions and their Fourier transforms that will be used in the analysis.  Then in \Cref{sec:weight-estimation}, we present our machinery for estimating the Fourier mass $I^\star_{f^{(\ell)}}(v,A)$.  In \Cref{sec:freq-finding}, we then present our algorithm that makes use of this machinery to locate the hidden directions $v_i$.  In \Cref{sec:function-recovery}, we then show how to recover the function $f$ given estimates for the directions $v_i$.  Finally in \Cref{sec:putting-together}, we put everything together to prove our main theorems.

\section{Properties of Functions}\label{sec:basic-properties}

Fourier analysis is a ubiquitous tool and it will play an important role in our learning algorithm.  We begin with some basic definitions and properties.

\begin{definition}[Fourier Transform]\label{def:fourier-transform}
For a function $f: \R \rightarrow \C$, we define its Fourier transform $\wh{f}: \R \rightarrow C$ as 
\[
\wh{f}(y) = \frac{1}{\sqrt{2\pi}}\int_{-\infty}^{\infty} e^{-\ii yx} f(x) dx \,.
\]
For a multivariable function $f: \R^d \rightarrow \C$, its Fourier transform $\wh{f}: \R^d \rightarrow C$ is defined as 
\[
\wh{f}(y) = \frac{1}{(2\pi)^{d/2}}\int_{\R^d} e^{-\ii y^\top x} f(x) dx \,.
\]
\end{definition}

\begin{fact}[Parseval's Identity]
For a function $f: \R^d \rightarrow \C$ such that $|f(x)|^2$ is integrable, we have
\[
\int_{\R^d} |\wh{f}(y)|^2 dy = \int_{\R^d} |f(x)|^2 dx \,.
\]
\end{fact}

The idea of rewighting by a Gaussian distribution and the interplay between the Fourier transform and Gaussian reweighting will also be important in our analysis.  We use the following notation.

\begin{definition}\label{def:Gaussian-distribution}
For $\mu \in \R^d, \Sigma \in \R^{d \times d}$, we let $\calN_{\mu, \Sigma}$ denote the Gaussian distribution with mean $\mu$ and covariance $\Sigma$.  We use $\calN_{\mu, \Sigma}(x)$ to denote the density function of the Gaussian at a point $x$. 
\end{definition}

We will often consider the Gaussian weighting of a function, which we denote as follows.
\begin{definition}\label{def:Gaussian-smoothing}
For a function $f:\R^d \rightarrow \C$ and $\ell > 0$, we define $f^{(\ell)}(x) = f(x) e^{-\norm{x}^2/(2\ell^2)}$.
\end{definition}

\subsection{Basic Fourier Transform Bounds}

We begin by proving a few basic bounds on the Fourier transform of a function after reweighting.  First, we have the following $L^\infty$ bound.

\begin{fact}\label{fact:fourier-linfty}
Let $\sigma: \R \rightarrow \R$ be a function with $|\sigma(x)| \leq 1$ for all $x$.  For any $\ell > 0$, let $\sigma^{(\ell)}(x)$ be as defined in \Cref{def:Gaussian-smoothing}. Then for all $y \in \R$, $|\wh{\sigma^{(\ell)}}(y)| \leq \ell$.
\end{fact}
\begin{proof}
Using $|\sigma|\le 1$ and triangle inequality,
\[
|\wh{\sigma^{(\ell)}}(y)|
\;=\;\left|\frac{1}{\sqrt{2\pi}}\int_{-\infty}^\infty e^{-\ii yx} \sigma(x) e^{-x^2/(2\ell^2)} dx\right|
\;\le\;\frac{1}{\sqrt{2\pi}}\int_{-\infty}^\infty e^{-x^2/(2\ell^2)} dx
\;=\;\ell.
\]
\end{proof}

Next, we bound the Lipschitz constant of the Fourier transform.
\begin{claim}\label{claim:lipschitz-fourier}
If $\sigma: \R \rightarrow \R$ satisfies $|\sigma(x)| \leq 1$ for all $x$, then the Fourier transform $\wh{\sigma^{(\ell)}}$ is $\ell^2$-Lipschitz.
\end{claim}
\begin{proof}
Write
\[
\frac{d}{dy} \wh{\sigma^{(\ell)}}(y) = \frac{1}{\sqrt{2\pi}}\int_{-\infty}^\infty (-\ii x) e^{-\ii yx} \sigma(x) e^{-x^2/(2\ell^2)} dx.
\]
Now triangle inequality immediately implies
\[
\left\lvert \frac{d}{dy} \wh{\sigma^{(\ell)}}(y) \right\rvert \;\le\; \frac{1}{\sqrt{2\pi}}\int_{-\infty}^\infty |x| e^{-x^2/(2\ell^2)} dx \; \leq \; \ell^2.
\]

\end{proof}

Finally, we prove a quantitative bound showing that on a bounded interval we can reconstruct $\sigma$ using only the portion of $\wh{\sigma^{(\ell)}}$ on a finite frequency window.

\begin{claim}\label{claim:truncated-inversion}
Assume $\sigma:\R \rightarrow \R$ has $|\sigma(x)|\le 1$ for all $x$ and that $\sigma$ is $L$-Lipschitz. For any $\ell \geq R \geq 1$, and cutoff $B>0$, define
\[
\widetilde\sigma_B(x) \;\defeq\; e^{\frac{x^2}{2\ell^2}}\;\frac{1}{\sqrt{2\pi}}\int_{-B}^{B} e^{\ii y x}\, \wh{\sigma^{(\ell)}}(y)\,dy.
\]
Then for all $|x|\le R$,
\begin{equation}\label{eq:trunc-error}
\big|\,\sigma(x) - \widetilde\sigma_B(x)\,\big| \;\le\; \frac{2 L\sqrt{2\ell}}{\sqrt{B}} \,.
\end{equation}
In particular, given any $\eps>0$, choosing
\begin{equation}\label{eq:B-choice}
B \;\ge\; \frac{8L^2\ell}{\eps^2}
\end{equation}
ensures $\max_{|x|\le R}\,|\sigma(x)-\widetilde\sigma_B(x)|\le \eps$.
\end{claim}

\begin{proof}
By Fourier inversion, for every $x\in\R$,
\[
\sigma(x) \;=\; e^{\frac{x^2}{2\ell^2}}\;\frac{1}{\sqrt{2\pi}}\int_{-\infty}^{\infty} e^{\ii y x}\, \wh{\sigma^{(\ell)}}(y)\,dy.
\]
Therefore, for any $B>0$ and any $x\in\R$,
\[
\sigma(x) - \widetilde\sigma_B(x)
\;=\; e^{\frac{x^2}{2\ell^2}}\;\frac{1}{\sqrt{2\pi}}\int_{|y|>B} e^{\ii y x}\, \wh{\sigma^{(\ell)}}(y)\,dy,
\]
and thus, for $|x|\le R$,
\begin{equation}\label{eq:tail-L1}
\big|\,\sigma(x) - \widetilde\sigma_B(x)\,\big| \;\le\; \int_{|y|>B} \big|\wh{\sigma^{(\ell)}}(y)\big|\,dy.
\end{equation}
To bound the tail $L^1$ norm, apply Cauchy--Schwarz with the weight $|y|^{-1}$:
\[
\int_{|y|>B} \big|\wh{\sigma^{(\ell)}}(y)\big|\,dy
\;\le\; \Big(\int_{|y|>B} \frac{dy}{y^2}\Big)^{\!1/2}
\Big(\int_{\R} \big|\,y\,\wh{\sigma^{(\ell)}}(y)\,\big|^2 dy\Big)^{\!1/2}
\;=\; \sqrt{\frac{2}{B}}\;\big\|\,y\,\wh{\sigma^{(\ell)}}\,\big\|_{2}.
\]
By Parseval, $\|\,y\,\wh{\sigma^{(\ell)}}\,\|_2 = \|\,(\sigma^{(\ell)})'\,\|_2$. Using $|\sigma'|\le L$ almost everywhere and $|\sigma|\le 1$,
\[
\|(\sigma^{(\ell)})'\|_2
\;=\; \big\|\,\sigma'\,e^{-\frac{x^2}{2\ell^2}} - \frac{x}{\ell^2}\,\sigma\,e^{-\frac{x^2}{2\ell^2}}\,\big\|_2
\;\le\; \Big(\int L^2 e^{-\frac{x^2}{\ell^2}}dx\Big)^{\!1/2} \!+\; \frac{1}{\ell^2}\Big(\int x^2 e^{-\frac{x^2}{\ell^2}}dx\Big)^{\!1/2}.
\]
Computing the Gaussian integrals gives
\[
\|(\sigma^{(\ell)})'\|_2 \;\le\; 2 L\sqrt{\ell} \,.
\]
Combining with \eqref{eq:tail-L1} yields \eqref{eq:trunc-error}. The stated choice \eqref{eq:B-choice} of $B$ then guarantees the error is at most $\eps$ on $[-R,R]$.
\end{proof}

\subsection{Non-degeneracy for Univariate Functions}

In our learning algorithm, our goal will be to identify all of the directions $v_i$. However, if the corresponding $\sigma_i$ is a constant function, then this is impossible.  To specify the set of directions that we will be able to uniquely identify, we introduce the following quantitative notion of non-degeneracy for univariate functions.

\begin{definition}\label{def:non-degenerate-function}
We say a function $\sigma:\R \rightarrow \R$ is $(R,\eps)$-nondegenerate if there are $R,\eps > 0$ such that there are $x_1, x_2 \in [-R,R]$ with $\sigma(x_1) - \sigma(x_2) \geq \eps$.
\end{definition}

We prove the following lemma which lower bounds the Fourier weight of a non-degenerate function. The point of this lemma is that it implies (see \Cref{coro:nonzero-weight-region-log}) that if a function is non-degenerate, then a non-trivial portion of the Fourier weight of its Gaussian reweighting must lie in a certain frequency band that is both bounded away from zero and from infinity.  The bounds on this region will be important in our learning algorithm later on.

\begin{lemma}\label{lem:nondegen-fourier}
Let $\sigma: \R \rightarrow \R$ be a function with $|\sigma(x)| \leq 1$ for all $x$.  For any $\ell > 0$, let $\sigma^{(\ell)}(x)$ be as defined in \Cref{def:Gaussian-smoothing}. Assume additionally that $\sigma$ is $L$-Lipschitz and $(R,\eps)$-nondegenerate where $R \geq 1, \eps < 1$, and that $\ell \geq 20(R + L)/\eps$. Then
\[
\int_{-\infty}^{\infty}  |\wh{\sigma^{(\ell)}}(y)|^2\, |y|^2\, e^{-y^2/\ell^2}\, dy \;\geq\; \frac{\eps^2}{8R}\,.
\]
\end{lemma}

\begin{proof}
Define
\[
h(x,\ell) \;\defeq\; \int_{-\infty}^{\infty} \frac{\ell}{\sqrt{2\pi}}\; \sigma^{(\ell)}(x - z)\; e^{-z^2 \ell^2/2}\, dz
 \,,
\]
where view $h$ as a function of $x$. A direct calculation shows
\[
\wh{h}(y,\ell)= \wh{\sigma^{(\ell)}}(y)\,e^{-y^2/(2\ell^2)}.
\]

By the assumed choice of $\ell$ and the assumption that $\sigma$ is $L$-Lipschitz, standard Gaussian estimates give
$\sup_{x\in[-R,R]} |h(x,\ell)-\sigma(x)| \le 0.1\,\eps$. Hence, there must be $x_1, x_2 \in [-R,R]$ such that
\[
h(x_1,\ell)-h(x_2,\ell) \;\ge\; \eps - 2\cdot 0.1\eps \;=\;0.8\eps \;\ge\;\eps/2.
\]
Therefore $\int_{-R}^R |h'(x,\ell)|\,dx \ge \eps/2$, and by Cauchy--Schwarz,
\[
\int_{-R}^R |h'(x,\ell)|^2\,dx \;\ge\; \frac{(\eps/2)^2}{2R} \;=\; \frac{\eps^2}{8R}.
\]
By Parseval,
\[
\int_{-\infty}^{\infty} |\wh{h'}(y,\ell)|^2\,dy
\;=\; \int_{-\infty}^{\infty} |h'(x,\ell)|^2\,dx
\;\ge\; \frac{\eps^2}{8R}.
\]
Using $\wh{h'}(y,\ell)=\ii y\,\wh{h}(y,\ell)=\ii y\, e^{-y^2/(2\ell^2)} \wh{\sigma^{(\ell)}}(y)$, we obtain
\[
\int_{-\infty}^{\infty}  |\wh{\sigma^{(\ell)}}(y)|^2\, |y|^2\, e^{-y^2/\ell^2}\, dy
\;=\; \int_{-\infty}^{\infty} |\wh{h'}(y,\ell)|^2\,dy
\;\ge\; \frac{\eps^2}{8R},
\]
as claimed.
\end{proof}

Using \Cref{lem:nondegen-fourier}, we can now prove the following corollary which lower bounds the Fourier weight of a non-degenerate function in a bounded frequency band that is also bounded away from zero.

\begin{corollary}\label{coro:nonzero-weight-region-log}
Assume $\sigma: \R \rightarrow \R$ satisfies $|\sigma(x)|\le 1$ and $\sigma$ is $L$-Lipschitz and $(R,\eps)$-nondegenerate where $R \geq 1, \eps < 1$, and let $\ell$ be some parameter with $\ell\ge 20(R+L)/\eps$. Define 
\[
a \;\defeq\; \frac{\eps}{8}\sqrt{\frac{1}{R\,\ell}}
\qquad\text{and}\qquad
B \;\defeq\; 4 \ell\log\!\Big(\frac{\ell R}{\eps}\Big).
\]
Set $\calS^{(\ell)} \defeq [-B,-a]\cup[a,B]$. Then
\[
\int_{\calS^{(\ell)}}\big|\widehat{\sigma^{(\ell)}}(y)\big|^2\,dy
\;\ge\; \frac{\eps^2}{8R\,\ell^2}.
\]
\end{corollary}

\begin{proof}
Throughout, write $W(y)\defeq y^2 e^{-y^2/\ell^2}$. From \Cref{lem:nondegen-fourier}, we have
\begin{equation}\label{eq:weighted-lb}
\int_{\R} |\widehat{\sigma^{(\ell)}}(y)|^2\, W(y)\,dy \;\ge\; \frac{\eps^2}{8R}.
\end{equation}
Decompose $\R$ into the ``good'' region $\calS^{(\ell)}$ and the ``bad'' region
$\calB^{(\ell)}\defeq (-a,a)\cup\{\,|y|>B\,\}$. Then
\[
\int_{\R} |\widehat{\sigma^{(\ell)}}|^2 W
= \int_{\calS^{(\ell)}} |\widehat{\sigma^{(\ell)}}|^2 W \;+\; \int_{\calB^{(\ell)}} |\widehat{\sigma^{(\ell)}}|^2 W.
\]
Let
\[
\alpha \;\defeq\; \sup_{y\in \calS^{(\ell)}} W(y),\qquad
\beta \;\defeq\; \sup_{y\in \calB^{(\ell)}} W(y),\qquad
M \;\defeq\; \int_{\R} |\widehat{\sigma^{(\ell)}}(y)|^2\,dy.
\]
By the definition of $\sigma^{(\ell)}$, Parseval, and $|\sigma(x)|\le 1$,
\begin{equation}\label{eq:M-bound}
M \;=\; \int_{\R} |\sigma^{(\ell)}(x)|^2 dx \;\le\; \int_{\R} e^{-x^2/\ell^2}\,dx \;=\; \ell\sqrt{\pi}.
\end{equation}

\paragraph{Bounding $\alpha$.}
Since $W$ is even, increases on $[0,\ell]$, and decreases on $[\ell,\infty)$, and
$\ell\in[a,B]$ (because $a\ll\ell$ and $B\ge 4\ell$), we have
\begin{equation}\label{eq:alpha}
\alpha \;=\; W(\ell) \;=\; \frac{\ell^2}{e}.
\end{equation}

\paragraph{Bounding $\beta$.}
On $(-a,a)$, $W(y)\le a^2$. On the tails $|y|>B$ we use monotonicity of $W$ on $[\ell,\infty)$:
\[
\sup_{|y|>B} W(y) \;=\; W(B)
\;=\; B^2 e^{-B^2/\ell^2}.
\]
By the choice of $B$, we have
\[
W(B) \;=\; 16 \ell^2 \log^2\!\Big(\frac{\ell R}{\eps}\Big) e^{-16 \log^2\!\Big(\frac{\ell R}{\eps}\Big)} \leq \frac{\eps^2}{64R\ell} \,.
\]

Therefore
\begin{equation}\label{eq:beta}
\beta \;\le\; \max\{\,a^2,\ W(B)\,\} 
\;\le\; \frac{\eps^2}{64R\ell} \,.
\end{equation}

\paragraph{Assembling the bounds.}
From \eqref{eq:weighted-lb} and the definitions,
\[
\frac{\eps^2}{8R} \;\le\; \int_{\calS^{(\ell)}} |\widehat{\sigma^{(\ell)}}|^2 W \;+\; \int_{\calB^{(\ell)}} |\widehat{\sigma^{(\ell)}}|^2 W
\;\le\; \alpha \int_{\calS^{(\ell)}} |\widehat{\sigma^{(\ell)}}|^2 \;+\; \beta \int_{\calB^{(\ell)}} |\widehat{\sigma^{(\ell)}}|^2
\;\le\; (\alpha-\beta)\!\int_{\calS^{(\ell)}}\! |\widehat{\sigma^{(\ell)}}|^2 \;+\; \beta M.
\]
Hence
\begin{equation}\label{eq:key-rearrange}
\int_{\calS^{(\ell)}} |\widehat{\sigma^{(\ell)}}(y)|^2\,dy
\;\ge\;
\frac{\displaystyle \frac{\eps^2}{8R} \;-\; \beta M}{\alpha-\beta}.
\end{equation}
Using \eqref{eq:M-bound} and \eqref{eq:beta},
\[
\beta M \;\le\; \left(\frac{\eps^2}{64R\ell} \right) \ell\sqrt{\pi}
\;=\; \frac{\sqrt{\pi}}{64}\,\frac{\eps^2}{R} \,.
\]
Next, the denominator clearly satisfies $\alpha - \beta \leq \alpha = \ell^2/e$.  Substituting everything back into \eqref{eq:key-rearrange}:
\[
\int_{\calS^{(\ell)}} |\widehat{\sigma^{(\ell)}}(y)|^2\,dy
\;\ge\;
\frac{\displaystyle \frac{\eps^2}{16R}}
{\displaystyle \cdot\frac{\ell^2}{e}}
\; \geq \;
\frac{1}{8}\,\frac{\eps^2}{R\,\ell^2},
\]
which is the claimed bound.
\end{proof}

\subsection{Fourier Transform of Ridge Functions}

So far in this section, we have focused on bounds for univariate functions.  The following formula for the Fourier transform of a Gaussian reweighted ridge function in high dimensions will be important for helping us reduce to a setting where we can apply our univariate estimates.

\begin{claim}\label{claim:fourier-formula}
Let $\sigma:\R \rightarrow \R$ be a function.  Define $f:\R^d \rightarrow \R$ as $f(x) = \sigma(v^\top x)$ for some unit vector $v \in \R^d$.  Then the Fourier transform of $f^{(\ell)}(x)$ is given by
\[
\wh{f^{(\ell)}}(y) = \wh{\sigma^{(\ell)}}(v^\top y)\cdot \ell^{d-1} e^{-\ell^2\norm{y - (v^\top y)v}^2/2}.
\]
\end{claim}
\begin{proof}
Extend the unit vector $v$ to an orthonormal basis of $\R^d$. Write any $x\in\R^d$ as $x = t\,v + z$ with $t\in\R$ and $z\in v^\perp$, and decompose $y$ as $y = \alpha v + w$ where $\alpha \defeq v^\top y$ and $w \defeq y - (v^\top y)v \in v^\perp$. Then $y^\top x = \alpha t + w^\top z$ and $\|x\|^2 = t^2 + \|z\|^2$. By \Cref{def:fourier-transform},
\[
\widehat{f^{(\ell)}}(y)
= \frac{1}{(2\pi)^{d/2}}
\int_{\R}\!\int_{v^\perp}
e^{-\ii(\alpha t + w^\top z)}\,
\sigma(t)\,e^{-\frac{t^2}{2\ell^2}}\,
e^{-\frac{\|z\|^2}{2\ell^2}}
\,dz\,dt.
\]
The integrals factor:
\[
\widehat{f^{(\ell)}}(y)
= \frac{1}{(2\pi)^{d/2}}
\left(\int_{\R} e^{-\ii\alpha t}\,\sigma(t)\,e^{-\frac{t^2}{2\ell^2}} dt\right)
\left(\int_{v^\perp} e^{-\ii w^\top z}\,e^{-\frac{\|z\|^2}{2\ell^2}} dz\right).
\]
The first bracket equals $\sqrt{2\pi}\,\widehat{\sigma^{(\ell)}}(\alpha)$. The second bracket is the standard Gaussian Fourier transform on $v^\perp \cong \R^{d-1}$, meaning
\[
\int_{z \in v^\perp} e^{-\ii w^\top z}\,e^{-\frac{\|z\|^2}{2\ell^2}} dz
= (2\pi)^{\frac{d-1}{2}}\,\ell^{\,d-1}\,e^{-\frac{\ell^2\|w\|^2}{2}}.
\]
Combining and noting that the $(2\pi)$ factors cancel gives
\[
\widehat{f^{(\ell)}}(y)
= \widehat{\sigma^{(\ell)}}(v^\top y)\;\ell^{\,d-1}\;e^{-\frac{\ell^2\|\,y - (v^\top y)\,v\,\|^2}{2}}
\]
as desired.
\end{proof}

\noindent We also have the following basic bound on the Lipschitzness of the Fourier transform in higher dimensions.

\begin{claim}\label{claim:fourier-lipchitz-2}
Let $f:\R^d \rightarrow \C$ be a function that has $|f(x)| \leq 1$ for all $x$. Then $\wh{f^{(\ell)}}$ is $\ell^{d+1}$-Lipschitz.
\end{claim}
\begin{proof}
We can write
\[
\nabla \wh{f^{(\ell)}}(y) = \frac{1}{(2\pi)^{d/2}}\int_{\R^d} (-\ii x) e^{-\ii y^\top x} f(x) e^{-\norm{x^2}/(2\ell^2)} \,.
\]
For any unit vector $v$, we have by triangle inequality,
\[
\left\lvert \left\langle v, \norm{\nabla \wh{f^{(\ell)}}(y)} \right\rangle \right\rvert \leq  \frac{1}{(2\pi)^{d/2}}\int_{\R^d}  e^{-\norm{x^2}/(2\ell^2)} |\langle v, x \rangle| \leq \ell^{d+1}
\]
where the last step follows from \Cref{claim:lipschitz-fourier} and the fact that the integral factorizes over $v$ and $v^\perp$.  This implies $\norm{\nabla \wh{f^{(\ell)}}(y)} \leq \ell^{d+1}$ as desired.
\end{proof}

\section{Estimating Fourier Weight}\label{sec:weight-estimation}

An important subroutine in our algorithm will be using queries to estimate the Fourier mass of the rescaled function $f^{(\ell)}$ (recall \Cref{def:Gaussian-smoothing}) around a point $v\in\R^d$, reweighted by a Gaussian with covariance $A^{-1}$ for some positive definite matrix $A$.

The first part of the analysis in this section will hold for a generic function $g:\R^d\to\R$; later we will apply it to the specific case $g=f^{(\ell)}$.  We begin with the following definition of the weighted Fourier mass.
\begin{definition}\label{def:fourier-weight}
For $g:\R^d\to\R$, we define the quantity 
\[
I^\star_{g}(v,A) \defeq \int_{\R^d} \left|\wh{g}(y)\right|^2 e^{-(y-v)^\top A (y-v)}\,dy,
\]
\end{definition}

Crucial to estimating this quantity is the following formula that rewrites $I^\star_{g}(v,A)$ as a quadratic form in $g$ itself, rather than its Fourier transform. This will allow us to estimate $I^\star_{g}(v,A)$ via sampling.  The crucial point about this formula is that it can be interpreted as an integral over $x$, of an expectation over a distribution of $\Delta$, of  $g(x)g(x + \Delta)$ times  a unit complex rotation, and the distribution is independent of $g$.

\begin{claim}\label{claim:gaussian-reweight}
Let $g: \R^d \rightarrow \R$ be such that $|g|,|g|^2$ are integrable. Then for any $v \in \R^d$ and $A \in \R^{d\times d}$ with $A\succ 0$,
\[
I^\star_{g}(v,A)
= \int_{\R^d} \left(e^{-\ii v^\top \Delta}\int_{\R^d} g(x+\Delta )g(x)\,dx\right)\,
N_{0,2A}(\Delta)\,d\Delta \,.
\]
\end{claim}

\begin{proof}
Recall that since $\wh{g}(y)=(2\pi)^{-d/2}\int g(x)e^{-\ii y^\top x}dx$, we have
\[
|\wh g(y)|^2=(2\pi)^{-d}\!\int\!\!\int g(x_1)g(x_2)e^{-\ii y^\top(x_1-x_2)}dx_1dx_2.
\]
Multiply by $w(y)\defeq e^{-(y-v)^\top A(y-v)}$ and integrate in $y$:
\[
\int |\wh g(y)|^2 w(y)dy=(2\pi)^{-d}\!\int\!\!\int g(x_1)g(x_2)\,\mathcal{I}(x_1-x_2)\,dx_1dx_2,
\]
where
\[
\mathcal{I}(\Delta)=\int_{\R^d} e^{-\ii y^\top \Delta} w(y)\,dy
=(2\pi)^{d/2}\,\widehat{w}(\Delta).
\]
Shifting $y=z+v$ and using the standard Gaussian FT,
\[
\widehat{w}(\Delta)=(2\pi)^{-d/2} e^{-\ii v^\top \Delta}\!\int e^{-z^\top A z-\ii z^\top\Delta}dz
=(2\pi)^{-d/2} e^{-\ii v^\top \Delta}\,\pi^{d/2}(\det A)^{-1/2}e^{-\tfrac14\Delta^\top A^{-1}\Delta}.
\]
Hence
\[
\mathcal{I}(\Delta)=e^{-\ii v^\top \Delta}\,\pi^{d/2}(\det A)^{-1/2}e^{-\tfrac14\Delta^\top A^{-1}\Delta}.
\]
Changing variables $(x_1,x_2)=(x+\Delta,x)$ and recognizing
\[
N_{0,2A}(\Delta)=(4\pi)^{-d/2}(\det A)^{-1/2}e^{-\tfrac14\Delta^\top A^{-1}\Delta},
\]
we note the constants cancel:
\[
(2\pi)^{-d}\cdot \pi^{d/2}(\det A)^{-1/2} = (4\pi)^{-d/2}(\det A)^{-1/2}.
\]
Thus
\[
\int_{\R^d} |\wh g(y)|^2 e^{-(y-v)^\top A(y-v)}dy
=\int_{\R^d} \Big(e^{-\ii v^\top \Delta}\int_{\R^d} g(x+\Delta)g(x)\,dx\Big)\,N_{0,2A}(\Delta)\,d\Delta.
\]
\end{proof}

Recall that in our original learning setup, we have query access to a function $f_{\sim}:\R^d\to\R$ that is close to $f$ in $L_\infty$ norm.  We now present our algorithm that makes use of \Cref{claim:gaussian-reweight} to estimate $I^\star_{f^{(\ell)}}(v,A)$ using queries to $f_{\sim}$.  In order to implement this, there is one additional wrinkle that we now explain.  Since \Cref{claim:gaussian-reweight} requires integrating $g(x)g(x+\Delta)$ over all $x$, we will set $g$ to be $f^{(\ell)}$ to ensure integrability.  Because we have query access to $f_{\sim}$, which is close to $f$, we can then rewrite the integral over all $x$ of  $f^{(\ell)}(x)f^{(\ell)}(x + \Delta)$  as an expectation over $x$ drawn from an appropriate Gaussian distribution of $f(x)f(x + \Delta)$, which we can then estimate by sampling.

\begin{algorithm2e}[ht!]\label{alg:estimate-Fourier}
\caption{Estimating Reweighted Fourier Mass}
\DontPrintSemicolon
\KwIn{Query access to $f_{\sim}:\R^d\to\R$}
\KwIn{$\ell>0$, vector $v\in\R^d$, matrix $A \in \R^{d \times d}$ with $A\succeq 0$, sample budget $m$}

\Fn{$\textsf{EstWeight}_{f_{\sim}, \ell,  m}(v, A)$}{
Set $C \leftarrow (\pi \ell^2)^{\frac d2}$.\;

\For{$j=1$ \KwTo $m$}{
  Draw $\Delta \leftarrow N_{0, 2A}$\;
  Draw $Z \sim N_{0,\tfrac{\ell^2}{2}I_d}$\;
  $x_- \leftarrow Z - \tfrac{1}{2}\Delta$; \quad $x_+ \leftarrow Z + \tfrac{1}{2}\Delta$\;
  $u \leftarrow f_{\sim}(x_-)$; \quad $w \leftarrow f_{\sim}(x_+)$\;
  $\phi \leftarrow e^{-\tfrac{\|\Delta\|^2}{4\ell^2} -\ii v^\top \Delta }$\;
  Set $c_j \leftarrow \phi\cdot u\cdot w$\;
}
\KwRet $I \leftarrow C\cdot \frac{c_1 + \dots + c_m}{m}$\;
}
\end{algorithm2e}

The following corollary shows that the output of Algorithm~\ref{alg:estimate-Fourier} is indeed a good estimate of the desired weighted Fourier mass, provided that $f_{\sim}$ is sufficiently close to $f$ and that we use enough samples.

\begin{corollary}\label{coro:concentration}
Let $f: \R^d \to \R$ be $L$-Lipschitz with $|f(x)|\le 1$, and fix $\ell>0$. Let $f^{(\ell)}$ be as defined in \Cref{def:Gaussian-smoothing}. Given query access to $f_{\sim}$ with $\|f-f_{\sim}\|_\infty\le \eps$ and also parameters $v \in \R^d, A \in \R^{d \times d}$ and sample budget $m$, compute $I = \textsf{EstWeight}_{f_{\sim}, \ell, m }(v, A)$ as defined in Algorithm~\ref{alg:estimate-Fourier}. For any $\delta\in(0,1)$, with probability at least $1-\delta$,
\[
\left|\, I - I^\star_{f^{(\ell)}}(v,A) \right|
\;\le\; 8 \,(\pi \ell^2)^{\frac d2}\!\left(\eps + \sqrt{\frac{2\log(8/\delta)}{m}}\right).
\]
\end{corollary}
\begin{proof}
Let $C\defeq(\pi\ell^2)^{d/2}$. For each $j \in [m]$, let $c_j$ be the value calculated by Algorithm~\ref{alg:estimate-Fourier} using $f_{\sim}$ and let $c_j^*$ be what the value would have been if calculated using $f$ instead. Since $|f|\le1$ and $\|f-f_{\sim}\|_\infty\le\eps$, we have
\[
|c_j|\le (1+\eps)^2,\qquad |c_j-c_j^*|\le \eps(2+\eps).
\]
Taking expectations and using $|\phi|\le 1$ yields
\[
\big|\E[c_j]-\E[c_j^*]\big|\le \eps(2+\eps).
\]
By \Cref{claim:gaussian-reweight}, 
\[
\begin{split}
C \E[c_j^*] &= \int \int f\left(Z - \frac{1}{2}\Delta\right) f\left(Z + \frac{1}{2}\Delta\right) e^{-\tfrac{\|\Delta\|^2}{4\ell^2} -\ii v^\top \Delta } (\pi \ell^2)^{\frac{d}{2}} N_{0, \frac{\ell^2}{2}I_d}(Z)  \cdot N_{0,2A}(\Delta) dZ d\Delta  \\ &= \int \int f^{(\ell)}\left(Z - \frac{1}{2}\Delta \right) f^{(\ell)}\left(Z + \frac{1}{2}\Delta \right)e^{-\ii v^\top \Delta } N_{0,2A}(\Delta) dZ d\Delta 
 \\ & = I^\star_{f^{(\ell)}}(v,A).
\end{split}
\]
Hence if $I$ is the output of Algorithm~\ref{alg:estimate-Fourier},
\[
\Big|\E[I]- I^\star_{f^{(\ell)}}(v,A)\Big|
\le C\,\eps(2+\eps).
\]
For sampling error, we can simply apply Hoeffding's inequality to the real and imaginary parts (each bounded in magnitude by $(1 + \eps)^2$) and union bound:
\[
\Pr\!\left[\left|\frac{1}{m}\sum_{j=1}^m \left( c_j - \E[c_j] \right) \right|\ge 2(1+\eps)^2\sqrt{\frac{2\log(8/\delta)}{m}}\right]\le \delta.
\]
Multiplying by $C$ and using $\eps < 1$ and simplifying gives the stated concentration bound.
\end{proof}

\section{Frequency Finding Algorithm}\label{sec:freq-finding}

In this section, our goal is to present an algorithm that makes queries to $f_{\sim}$ that is $\eps$-close in $L_{\infty}$ to a function
\[
f(x) = \sigma_1(v_1^\top x) + \dots + \sigma_n(v_n^\top x)
\]
and finds the directions $v_1 , \dots , v_n$ (recall that throughout this paper we will maintain the convention that the $v_i$ are unit vectors).  However, this exact goal is not possible since if there is some $\sigma_i$ that is constant, then we don't be able to recover the direction $v_i$.  However, we will show how to recover all of the directions where $\sigma_i$ is non-degenerate, and this will suffice downstream for reconstructing the function. For a precise statement, see \Cref{lem:find-directions}, which is the main result that we will prove in this section.

Before presenting the algorithm, we define some notation and prove a few facts that will be useful in the analysis.  First, we define the following oracle, based on \Cref{coro:concentration}, that will simplify the exposition.

\begin{definition}[Fourier Mass Oracle]\label{def:fourier-oracle}
A \emph{Fourier Mass Oracle} is parameterized by an underlying function $g$ and accuracy $\tau$.  The oracle $\calI_{\tau, g}$ takes as input a vector $v \in \R^d$ and a positive semidefinite matrix $A \succeq 0$ and outputs an estimate of $\calI_{\tau, g}(v,A)$ such that
\[
|I^\star_{g}(v,A) - \calI_{\tau, g}(v,A)| \leq \tau \,.
\]
We call such an oracle a $\tau$-accurate oracle for the function $g$.
\end{definition}

In light of \Cref{coro:concentration}, we can implement a Fourier Mass Oracle for the function $f^{(\ell)}$ with accuracy $\tau = 10\eps (\pi \ell^2)^{d/2}$ and exponentially small failure probability using polynomially many queries to $f_{\sim}$.  For the rest of this section, we will only interact with $f_{\sim}$ via such an oracle, and thus we will measure query complexity in terms of the number of calls to such an oracle rather than direct queries to $f_{\sim}$.  We will bound the actual query complexity in terms of the number of oracle calls when we put everything together in  \Cref{sec:putting-together}.

\subsection{Location of Nonzero Frequencies}

Next, we prove two statements, \Cref{lem:where-zero} and \Cref{lem:where-nonzero}, which characterize the $v,A$ for which the Fourier mass $I^\star_{f^{(\ell)}}(v,A)$ can be non-negligible. First, we show that when $v$ is far (in the norm induced by $A$) from every hidden direction line $\{t v_i: t\in\R\}$, then the Fourier mass is small.  To interpret the bound in \Cref{lem:where-zero}, we will set $A$ so that $\norm{A} \leq \ell^2$, so then the Fourier mass $I^\star_{f^{(\ell)}}(v,A)$ decays exponentially in $D^2$ where $D$ is the $A$-distance from $v$ to the union of lines $\{tv_i : t \in \R \}_{i \in [n]}$.

\begin{lemma}[Where weight is negligible]\label{lem:where-zero}
Let $f(x) = \sigma_1(v_1^\top x) + \dots + \sigma_n(v_n^\top x)$ be a sum of features satisfying $\norm{\sigma_i}_{\infty} \leq 1 \; \forall i$.  Let $A\succeq 0$ be any positive semidefinite matrix. Define the $A$-distance from $v$ to the union of lines by
\[
 D^2 \;\defeq\; \min_{i\in[n]}\;\min_{t\in\R}\; (v - t v_i)^\top A\,(v - t v_i).
\]
Then
\begin{equation}\label{eq:where-zero-bound}
I^\star_{f^{(\ell)}}(v,A) \;\le\; n^2\,\pi^{\frac d2}\,\ell^{d}\;\exp\!\Big( -\,\frac{\ell^2}{\ell^2 + \|A\|}\, D^2\Big) \,.
\end{equation}
\end{lemma}
\begin{proof}
We use the shorthand $I^\star \defeq I^\star_{f^{(\ell)}}(v,A)$. Using \Cref{claim:fourier-formula} and Cauchy--Schwarz,
\[
|\wh{f^{(\ell)}}(y)|^2 \;\le\; n\sum_{i=1}^n \big|\wh{\sigma_i^{(\ell)}}(v_i^\top y)\big|^2\,\ell^{2d-2}\, e^{-\ell^2\|y-(v_i^\top y)v_i\|^2}.
\]
First consider a fixed $i$ and decompose $y=t v_i+z$ with $z\in v_i^\perp$. For $x\defeq t v_i - v$, write $P$ for the orthogonal projector onto $v_i^\perp$, set $A_{\perp}\defeq P A P$, and define $ M\;\defeq\; \ell^2 I_{\perp}+A_{\perp}$. Note that when we compute 
\[
I^\star = \int |\wh{f^{(\ell)}}(y)|^2 e^{-(y-v)^\top A (y-v)} dy \,,
\]
and substitute in the above bound on $|\wh{f^{(\ell)}}(y)|^2$, we will obtain an expression with the following quadratic in the exponent
\[
\ell^2\|y-(v_i^\top y)v_i\|^2  + (y-v)^\top A (y-v) = \ell^2\|z\|^2 + (z+x)^\top A(z+x) \,.
\]
We will then first integrate over $z\in v_i^\perp$ and then over $t\in\R$. Since $A_{\perp}\succeq 0$ and $\ell>0$, we have $M\succ 0$ as an operator on the $d-1$ dimensional space $v_i^\perp$. Observe that
\begin{align*}
\int_{v_i^\perp} e^{-\ell^2\|z\|^2}\, e^{-(z+x)^\top A (z+x)}dz
 = \frac{\pi^{\frac{d-1}{2}}}{\sqrt{\det M}}\;\exp\!\Big( -\,\phi_{\perp}(x)\Big),
\end{align*}
where
\[
\phi_{\perp}(x)\;\defeq\; \min_{z\in v_i^\perp}\Big\{\;\ell^2\|z\|^2 + (z+x)^\top A(z+x)\;\Big\} 
\]
and again $M$ is viewed as an operator on $v_i^\perp$ so its determinant is positive. To see why the above characterization as a minimum holds, note that the integral is over a rescaled Gaussian with covariance matrix $(2M)^{-1}$ and thus evaluating the ``scaling factor'' in the integral is the same as computing the maximum value of the quadratic form in the exponent.


To upper bound the integral, we can relax $z\in v_i^\perp$ to $z\in\R^d$; then
\[
\phi_{\perp}(x)\;\ge\; \min_{z\in\R^d}\Big\{\;\ell^2\|z\|^2 + (z+x)^\top A(z+x)\;\Big\}
\;=\; \ell^2\, x^\top A\,(\ell^2 I + A)^{-1} x,
\]
where the minimizer is $z_* = - (\ell^2 I + A)^{-1} A x$; note that $\ell^2 I + A\succ 0$, so the inverse is well-defined even if $A$ is singular.  Also note that $A$ and $(\ell^2 I + A)^{-1}$ commute which allows us to write the expression in the above form.  Consequently,
\[
\int_{v_i^\perp} e^{-\ell^2\|z\|^2}\, e^{-(z+x)^\top A (z+x)}dz
\;\le\; \frac{\pi^{\frac{d-1}{2}}}{\sqrt{\det(\ell^2 I_{\perp}+A_{\perp})}}\;\exp\!\Big( -\,\ell^2\,x^\top A\,(\ell^2 I + A)^{-1} x\Big).
\]
Since $\det(\ell^2 I_{\perp}+A_{\perp})\ge \ell^{2(d-1)}$ and $(\ell^2 I + A)^{-1}\succeq \tfrac{1}{\ell^2+\|A\|}I$, we obtain
\[
\int_{v_i^\perp} e^{-\ell^2\|z\|^2}\, e^{-(z+x)^\top A (z+x)}dz\;\le\; \pi^{\frac{d-1}{2}}\,\ell^{-(d-1)}\,\exp\!\Big( -\,\frac{\ell^2}{\ell^2+\|A\|}\; x^\top A x\Big).
\]
Now we can put everything together and integrate over $t$ as well. Letting $D_i^2\defeq \min_{t\in\R}(v-t v_i)^\top A (v-t v_i)$ and $D^2=\min_i D_i^2$,
\[
\begin{split}
I^\star \;\le\; n\sum_{i=1}^n \pi^{\frac{d-1}{2}}\,\ell^{d-1}\!\int_{\R} \Big|\wh{\sigma_i^{(\ell)}}(t)\Big|^2\exp\!\Big( -\,\frac{\ell^2}{\ell^2+\|A\|}\; (v-t v_i)^\top A (v-t v_i)\Big)dt \\
\;\le\; n\sum_{i=1}^n \pi^{\frac{d-1}{2}}\,\ell^{d-1}\,e^{-\frac{\ell^2}{\ell^2+\|A\|}D_i^2}\!\int_{\R} \Big|\wh{\sigma_i^{(\ell)}}(t)\Big|^2 dt.
\end{split}
\]
By Parseval in one dimension and $|\sigma_i|\le 1$, we get $\int_{\R}|\wh{\sigma_i^{(\ell)}}(t)|^2 dt \le \ell\,\sqrt{\pi}$. It follows that
\[
I^\star \;\le\; n^2\,\pi^{\frac d2}\,\ell^{d}\;\exp\!\Big( -\,\frac{\ell^2}{\ell^2 + \|A\|}\, D^2\Big)
\]
and this completes the proof.
\end{proof}

Next, we prove a counterpart to Lemma~\ref{lem:where-zero}, showing that for each direction $v_i$ with a non-degenerate activation $\sigma_i$, there is a bounded scale $\beta$ with $\beta$ also bounded away from $0$ such that if $v$ is close to $\beta v_i$ and $A$ is not too large, then the Fourier mass $I^\star_{f^{(\ell)}}(v,A)$ is bounded away from zero.  

To interpret the bound in \Cref{lem:where-nonzero}, will ensure $\alpha \sim \ell^2/d$.  This means that the factor on the outside reduces to $(\pi \ell^2)^{\frac{d-1}{2}}$ (up to a constant).  As long as $\ell$ is sufficiently large, then the exponential on the inside becomes negligible and the inside is lower bounded by some inverse polynomial.  Thus, overall the expression will be lower bounded by some inverse polynomial times  $(\pi \ell^2)^{\frac{d}{2}}$ (which is the scaling factor that shows up naturally from the Fourier mass oracle).

\begin{lemma}[Where weight is non-negligible]\label{lem:where-nonzero}
Let $f(x) = \sigma_1(v_1^\top x) + \dots + \sigma_n(v_n^\top x)$ be a sum of features satisfying $ \norm{\sigma_i}_{\infty} \leq 1 \; \forall i$.  
Assume the directions are $\gamma$-separated i.e. the sines of all angles between them are at least $\gamma$. Fix $i\in[n]$ and assume $\sigma_i$ is $L$-Lipschitz and $(R,\Delta)$-nondegenerate where $R \geq 1, \Delta < 1$. Let $\ell\ge 20(R+L)/\Delta$, and define
\[
 a\;\defeq\; \frac{\Delta}{8}\sqrt{\frac{1}{R\,\ell}},\qquad B\;\defeq\;4\ell\,\log\!\Big(\frac{\ell R}{\Delta}\Big),\qquad E\;\defeq\;[a,B]\cup[-B,-a].
\]
There exists a $\beta\in E$ such that the following holds. Set $A=\alpha I_d$ and $v=\beta v_i$. Then
\begin{equation}\label{eq:where-nonzero-lb}
\begin{split}
I^\star_{f^{(\ell)}}(v,A) \;\ge\; \frac{\pi^{\frac{d-1}{2}}\,\ell^{2d-2}}{(\ell^2+\alpha)^{\frac{d-1}{2}}}\;\Bigg[\;\frac{\Delta^2}{64n (1 + B\sqrt{\alpha}) R\,\ell^2}
\;-
\; 4n\ell\;\exp\!\Big( -\,\frac{\alpha \ell^2 \gamma^2 a^2}{\ell^2+\alpha}\Big)\;\Bigg] \,.
\end{split}
\end{equation}
\end{lemma}
\begin{proof}
By \Cref{claim:fourier-formula} we can decompose $\wh{f^{(\ell)}}(y)=\sum_{j=1}^n T_j(y)$, where
\[
T_j(y)\;\defeq\; \wh{\sigma_j^{(\ell)}}\big(v_j^\top y\big)\,\ell^{d-1}\,e^{-\ell^2\|y-(v_j^\top y)v_j\|^2/2}.
\]
For any complex numbers $z_1,z_2 ,\dots , z_n$, $|z_1+z_2 + \dots z_n|^2\ge \frac{1}{n}|z_i|^2 - \sum_{j \neq i} |z_j|^2$. Applying this and integrating against the nonnegative weight $e^{-(y-v)^\top A(y-v)}$ gives
\begin{equation}\label{eq:split-lb}
I^\star_{f^{(\ell)}}(v,A)\;\ge\; \frac{1}{n} \, I^\star_i\; -\; \sum_{j\ne i} I^\star_j,
\end{equation}
where we write the per-component masses
\[
I^\star_j \;\defeq\; \int_{\R^d} \Big|\wh{\sigma_j^{(\ell)}}(v_j^\top y)\Big|^2\,\ell^{2d-2}\,e^{-\ell^2\|y-(v_j^\top y)v_j\|^2}\,e^{-(y-v)^\top A(y-v)}dy.
\]
We first lower bound the single-component contribution $I^\star_i$. Set $v=\beta v_i$ and decompose $y=t v_i+z$ with $z\in v_i^\perp$; then $\iprod{v,v_i}=\beta$. Recall $A=\alpha I_d$. Now we follow a similar calculation to Lemma~\ref{lem:where-zero} where we first integrate over $z \in v^{\perp}$ and then over $t$ to get
\[
I^\star_i = \int_{\R} | \wh{\sigma_{i}^{(\ell)}} (t)|^2 \; \ell^{2d-2} \int_{v_i^\perp}  e^{-(\ell^2 + \alpha) \norm{z}^2} e^{-\alpha(t - \beta)^2}  dz dt = \frac{\pi^{\frac{d-1}{2}}\,\ell^{2d-2}}{(\ell^2+\alpha)^{\frac{d-1}{2}}}\;\int_{\R} \Big|\wh{\sigma_i^{(\ell)}}(t)\Big|^2 e^{-\alpha(t-\beta)^2}dt.
\]
By \Cref{coro:nonzero-weight-region-log}, since $\sigma_i$ is $L$-Lipschitz and $(R,\Delta)$-nondegenerate and $\ell\ge 20(R+L)/\Delta$, we have
\[
\int_{E} \big|\wh{\sigma_i^{(\ell)}}(t)\big|^2 dt \;\ge\; \frac{\Delta^2}{8R\,\ell^2}.
\]
Let $E_+= [a,B]$ and $E_-=[-B,-a]$, and pick the sign $s\in\{\pm1\}$ maximizing $\int_{E_s}|\wh{\sigma_i^{(\ell)}}(t)|^2 dt$.  This guarantees we keep at least half of the total integral.  WLOG $E_s = E_+$ For any fixed $t\in E_+$, the integral over $\beta\in E_+$ of $e^{-\alpha(t-\beta)^2}$ is
\[
\int_{\beta\in E_+} e^{-\alpha(t-\beta)^2}d\beta \;=\; \int_{t-B}^{t-a} e^{-\alpha u^2}\,du 
\;\ge\; \frac{1}{4} \min\left(\frac{1}{\sqrt{\alpha}} , B \right).
\]
Averaging over $\beta\in E_+$ therefore shows that there exists a choice of $\beta\in E_+$ with
\[
\int_{\R} \big|\wh{\sigma_i^{(\ell)}}(t)\big|^2 e^{-\alpha(t-\beta)^2}dt \;\ge\; \frac{1}{4(1 + B\sqrt{\alpha})}\int_{E_+} \big|\wh{\sigma_i^{(\ell)}}(t)\big|^2 dt \;\ge\; \frac{\Delta^2}{64 (1 + B\sqrt{\alpha})R\,\ell^2}.
\]
Combining with the factor obtained from integrating over $z$, we get that there exists some choice of $\beta$ such that 
\[
I^\star_i\;\ge\; \frac{\pi^{\frac{d-1}{2}}\,\ell^{2d-2}}{(\ell^2+\alpha)^{\frac{d-1}{2}}}\;\frac{\Delta^2}{64 (1 + B\sqrt{\alpha})R\,\ell^2} \,.
\]
It remains to control the remainder $\sum_{j\ne i} I^\star_j$ in \eqref{eq:split-lb}.  Consider a fixed index $j$.  We apply the same approach of integrating over $v_j^{\perp}$ and then over $t$.  We set $A=\alpha I_d$ and $v=\beta v_i$ and also let $w_j = P_{v_j^\perp}v$ and get 
\[
\begin{split}
I^\star_j  & = \int_{\R} | \wh{\sigma_{j}^{(\ell)}} (t)|^2 \; \ell^{2d-2} \int_{v_i^\perp}  e^{-\ell^2 \norm{z}^2 - \alpha \norm{tv_j + z - v}^2 }  dz dt \\ & =   \int_{\R} | \wh{\sigma_{j}^{(\ell)}} (t)|^2 \int_{v_j^\perp} e^{-\ell^2 \norm{z}^2 - \alpha \norm{z - w_j}^2} e^{- \alpha \norm{tv_j - v + w_j}^2} dz dt
\\ & = \left( \int_{v_j^\perp} e^{-\ell^2 \norm{z}^2 - \alpha \norm{z - w_j}^2}  dz \right) \left(\int_{\R} | \wh{\sigma_{j}^{(\ell)}} (t)|^2 e^{- \alpha \norm{tv_j - v + w_j}^2}  dt \right)
\\ & = \frac{\pi^{\frac{d-1}{2}}\,\ell^{2d-2}}{(\ell^2+\alpha)^{\frac{d-1}{2}}}  \min_{z \in v_j^\perp}\left( e^{-\ell^2 \norm{z}^2 - \alpha \norm{z - w_j}^2}\right) \int_{\R} | \wh{\sigma_{j}^{(\ell)}} (t)|^2 e^{- \alpha \norm{tv_j - v + w_j}^2}  dt 
\\ & \leq \frac{\pi^{\frac{d-1}{2}}\,\ell^{2d-2}}{(\ell^2+\alpha)^{\frac{d-1}{2}}} \exp\!\Big( -\,\frac{\alpha\,\ell^2}{\ell^2+\alpha}\,\|w_j \|^2\Big) \int_{\R} | \wh{\sigma_{j}^{(\ell)}} (t)|^2 dt
\\ & \leq \frac{\pi^{\frac{d}{2}}\,\ell^{2d-1}}{(\ell^2+\alpha)^{\frac{d-1}{2}}} \exp\!\Big( -\,\frac{\alpha\,\ell^2}{\ell^2+\alpha}\,\|w_j \|^2\Big)
\end{split}
\]
where the last step follows from Parseval and the assumption on $\sigma_j$. By $\gamma$-separation, for unit vectors $v_i,v_j$, $\sin^2\!\angle(v_i,v_j) \ge \gamma^2$. Since $v=\beta v_i$, $\|w_j\|^2 = \beta^2 \sin^2\!\angle(v_i,v_j)\ge \beta^2\,\gamma^2 \ge a^2\,\gamma^2 $ for $\beta\in E$. Hence
\[
\sum_{j\ne i} I^\star_j \;\le\; (n-1)\,\frac{\pi^{\frac d2}\,\ell^{2d-1}}{(\ell^2+\alpha)^{\frac{d-1}{2}}}\;\exp\!\Big( -\,\frac{\alpha \ell^2 \gamma^2 a^2}{\ell^2+\alpha}\Big).
\]
Plugging back into \eqref{eq:split-lb} and  substituting the bounds we obtained into the expression
\[
I^\star_{f^{(\ell)}}(v,A) \;\ge\; \frac{1}{n} I^\star_i\; -\; \sum_{j\ne i} I^\star_j 
\]
gives the desired inequality

\end{proof}

\subsection{Direction Recovery Algorithm and Analysis}

Now we present our algorithm for recovering the hidden directions.  Recall, the only way the algorithm interacts with the unknown function $f$ is through a Fourier mass oracle (\Cref{def:fourier-oracle}).  

We begin with a high-level description of the algorithm. The algorithm takes as input some orthonormal basis $b_1, \dots , b_d$ as well as information about some of the coordinates, say $\alpha_1, \dots , \alpha_k$ for $k \leq d$.  The idea is to then search over possibilities for the next coordinate $\alpha_{k+1}$ such that there is nontrivial total Fourier mass on the set of all points close to $\alpha_1 , \dots , \alpha_{k+1}$ on their first $k+1$ coordinates.  We can then repeat and recurse to search for the coordinate $\alpha_{k+2}$ and so on.  Once we have fixed all $d$ coordinates, we simply return the unit vector in the direction $(\alpha_1, \dots , \alpha_d)$.

For the actual implementation, we have width parameters $K_1, \dots , K_{k+1}$ which are sufficiently large.  To localize around points that are close to $\alpha_1, \dots , \alpha_{k+1}$ in their first $k+1$ coordinates, we query $\calI_{\tau, f^{(\ell)}}(v, A)$ for 
\[
v = \alpha_1b_1 + \dots + \alpha_k b_k + \alpha_{k+1}b_{k+1} \quad \quad A = K_1 b_1b_1^\top + \dots +  K_{k+1} b_{k+1}b_{k+1}^\top \,.
\]
The reason we require different width parameters for the different coordinates is for technical details later on for bounding the branching factor in this algorithm.  The details of the algorithm are described below in Algorithm~\ref{alg:search}.  Note that the only parameters that change between levels of recursion are the current coordinates $\alpha_1, \dots , \alpha_k$ (since we fix an additional coordinate in reach iteration).  All other parameters are global, i.e. shared between all levels of recursion.


\begin{algorithm2e}[ht!]\label{alg:search}
\caption{Find Directions}
\DontPrintSemicolon
\KwIn{Width $\ell$, accuracy $\eps$ (global)}
\KwIn{Access to Fourier mass oracle $\calI_{\tau, f^{(\ell)}}$ with $\tau = 10\eps(\pi \ell^2)^{d/2}$}
\KwIn{Width parameters $C_1, C_2$ (global)}
\KwIn{Orthonormal basis $b_1, \dots , b_d \in \R^d$ (global)}
\KwIn{Current coordinates $\alpha_1, \dots , \alpha_k \in \R$ (where $k \leq d$) }

\If{$k = d$}{
\KwRet{$\frac{\alpha_1 b_1 + \dots + \alpha_d b_d}{\sqrt{\alpha_1^2 + \dots + \alpha_d^2}}$}\;
}
Let $\calT$ be the set of all integer multiples of $1/\sqrt{10C_2}$ between $-\ell^2$ and $\ell^2$ \;
\For{$c \in \calT$}{
Set $v = \alpha_1b_1 + \dots + \alpha_k b_k + cb_{k+1}$\;
Set $(K_1, \dots , K_{k+1}) = (C_2, C_2, C_1, \dots , C_1, C_2)$ \;
Set $A =K_1 b_1b_1^\top + \dots + K_{k+1} b_{k+1}b_{k+1}^\top$\;
Query $W_c = \calI_{\tau, f^{(\ell)}}(v, A)$\;
}
Let $\calT' = \{ c | c \in T , |W_c| \geq 5 \tau \}$ \;
Let $\calS$ be any maximal subset of $\calT'$  whose elements are $1/\sqrt{10dC_1}$-separated\;
\For{$\alpha_{k+1} \in \calS$}{
Recurse on $\alpha_1, \dots , \alpha_{k+1}$\;
}
\end{algorithm2e}

Now we are ready to analyze Algorithm~\ref{alg:search}.  First, we set a couple parameters.  We assume we are given parameters $d,n, R,L, \gamma, \eps$ which govern the properties of the unknown function
\[
f(x) =  \sigma_1(v_1^\top x) + \dots +  \sigma_n(v_n^\top x)
\]
as in \Cref{sec:setup}.  We also assume we are given a target accuracy parameter $\Delta$ and that $\eps < \frac{1}{\poly\left(d, n, R, L, \frac{1}{\gamma}, \frac{1}{\Delta} \right)}$ for some sufficiently large polynomial.

We will set the global parameters in Algorithm~\ref{alg:search} as follows:
\begin{equation}\label{eq:hyperparams}
\ell = \poly\left(d, n, R, L, \frac{1}{\gamma}, \frac{1}{\Delta} \right)  \; , \; C_2 = \frac{\ell^2}{d} \; , \; C_1 = C_2^{0.9} \,.
\end{equation}
but we ensure $\eps \ll 1/\poly(\ell)$, which is possible as long as $\eps$ is a sufficiently small inverse polynomial.

Now we begin by giving a geometric characterization that will be useful for the analysis.
\begin{definition}\label{def:2d-separated}
We say the vectors $b_1, b_2$ of the orthonormal basis are $\theta$-separating if 
\begin{itemize}
\item For every $i \in [n]$, $|v_i \cdot b_1|, |v_i \cdot b_2| \geq \theta/\sqrt{d}$
\item For every $i,j \in [n]$ with $i \neq j$, 
\[
(v_i \cdot b_1) (v_j \cdot b_2) - (v_j \cdot b_1) (v_i \cdot b_2) \geq \frac{\theta^2}{d} \,.
\]
\end{itemize}
\end{definition}

\Cref{def:2d-separated} is useful because it implies that if we fix any projection $(\alpha_1, \alpha_2)$ in the plane formed by $b_1, b_2$ that is not too close to the origin, then there is at most one $v_i$ such that the projection of $tv_i$ is very close to $\alpha_1 b_1 + \alpha_2 b_2$ for some $t \in \R$.  This will be crucial for arguing that Algorithm~\ref{alg:search} doesn't branch too much in the recursive step.  We now show that with high probability, a random orthonormal basis will be $\theta$-separating for $\theta$ not too small.

\begin{claim}\label{claim:anticoncentration}
Assume that $v_1, \dots , v_n$ are unit vectors such that the sines of the pairwise angles between them are all at least $\gamma$.  Then with $1 - \frac{1}{10n}$ probability over the choice of a random orthonormal basis $b_1,b_2, \dots , b_d$, we have that $b_1, b_2$ are $\frac{\gamma}{(10n)^3}$-separating.
\end{claim}
\begin{proof}
For the first condition, since $b_1,b_2$ are each individually uniform over the sphere, anti-concentration implies that for each $i$,
\[
\Pr[|v_i \cdot b_1| \leq \theta/\sqrt d] \leq 10 \theta 
\] 
and similar for $b_2$. Thus, we get with probability at least $1- 20n\theta$ that simultaneously $|v_i\cdot b_1|,|v_i\cdot b_2|\ge \theta/\sqrt d$ for all $i\in[n]$.
\\\\
For the second condition, fix $i\ne j$. Condition on $b_1$.  Write
\[
(v_i \cdot b_1)(v_j \cdot b_2) - (v_j \cdot b_1)(v_i \cdot b_2) = ((v_i\cdot b_1)v_j - (v_j\cdot b_1) v_i) \cdot b_2
\]
Note $(v_i\cdot b_1)v_j - (v_j\cdot b_1) v_i$ is orthogonal to $b_1$ and
\[
\norm{(v_i\cdot b_1)v_j - (v_j\cdot b_1) v_i} \geq \max(|v_i \cdot b_1|, |v_j \cdot b_1|) \cdot \sin\angle(v_i,v_j) \geq \max(|v_i \cdot b_1|, |v_j \cdot b_1|) \cdot \gamma \,.
\]
Thus, with probability at least $1 - 10\theta$ over the randomness of $b_2$, 
\[
(v_i \cdot b_1)(v_j \cdot b_2) - (v_j \cdot b_1)(v_i \cdot b_2) \geq \frac{\gamma \theta}{\sqrt{d}} \max(|v_i \cdot b_1|, |v_j \cdot b_1|) \,.
\]
Thus, setting $\theta = \frac{1}{(10n)^3}$ and taking a union bound over all $i,j$ completes the proof.
\end{proof}

Now we analyze Algorithm~\ref{alg:search} assuming that we initialize with $\alpha_1, \alpha_2$ that satisfy certain properties.  First, we show that if there is some nontrivial Fourier mass on the hyperplane through $\alpha_1b_1 + \alpha_2 b_2$ orthogonal to $b_1, b_2$, then the algorithm will succeed and return some direction that contains nontrivial Fourier mass. 

\begin{definition}
We say a point $v \in \R^d$ is heavy if $I^\star_{f^{(\ell)}}(v, 4C_2I_d) \geq 10^3\eps(\pi \ell^2)^{d/2}$.
\end{definition}

\begin{claim}\label{claim:find-good-point}
Assume that we run Algorithm~\ref{alg:search} starting with $\alpha_1, \alpha_2$ such that 
\begin{itemize}
\item $\alpha_1^2 + \alpha_2^2 \geq \frac{1}{C_1^{0.6}}$
\item There is some point $v$ with $|v \cdot b_1 - \alpha_1|, |v \cdot b_2 - \alpha_2| \leq \frac{1}{\sqrt{10C_2}} $ and $\norm{v} \leq \ell^2/2$ that is heavy
\end{itemize}  
Then the algorithm will return some point $u$ with 
\[
\norm{u - \frac{v}{\norm{v}}} \leq \frac{1}{C_1^{0.2}} \,.
\] 
\end{claim}
\begin{proof}
We prove by induction that for each $k \geq 3 $, the algorithm recurses on some coordinates $\alpha_1, \dots , \alpha_k$ such that
\[
\sum_{i = 3}^k (\alpha_i - v \cdot b_i)^2 \leq \frac{k-2}{5 d C_1} \,.
\]
Assume that this is true at level $k$ \---- the base case for $k = 2$ is trivial.  There must be some choice of $c \in \calT$ such that $|c - v \cdot b_{k+1}| \leq 1/\sqrt{10 C_2}$.  First, we claim that this choice of $c$ will be in the set of points  $\calT'$ (as constructed in the execution of Algorithm~\ref{alg:search}).  To see this, set 
\[
v_0 = \alpha_1 b_1 + \dots + \alpha_k b_k + c b_{k+1} \; , \;  A_0 = C_2 b_1b_1^\top + C_2 b_2 b_2^\top + C_1 b_3b_3^\top + \dots + C_1 b_kb_k^\top + C_2b_{k+1}b_{k+1}^\top
\]
and then since we assumed $v$ is heavy, we have
\[
I^\star_{f^{(\ell)}}(v_0 , A_0)    \geq 0.3 I^\star_{f^{(\ell)}}(v, 4C_2 I_d)  = 300\eps(\pi \ell^2)^{d/2}
\]
where the first inequality holds because if we set $\bar{v}$ to be the projection of $v$ onto the span of $b_1, \dots b_{k+1}$, then for all $y \in \R^d$,
\[
e^{-(v_0 - y)^\top A_0 (v_0 - y)} \geq  e^{-2(v_0 - \bar{v})^\top A_0 (v_0 - \bar{v})  - 2(\bar{v} - y)^\top A_0 (\bar{v} - y)} \geq  0.3 e^{-4(v - y)^\top C_2 I_d (v - y)} \,.
\]
Thus, by the guarantees of the Fourier mass oracle $\calI_{\tau, f^{(\ell)}}$ (and since we set $\tau = 10\eps(\pi \ell^2)^{d/2}$), this value of $c$ must be included in $\calT'$.

Now, when we filter down the set $\calT'$, we must include some $\alpha_{k+1}$ in the set $\calS$  such that 
\[
|\alpha_{k+1} - v \cdot b_{k+1}| \leq \frac{1}{\sqrt{10 d C_1}} + \frac{1}{\sqrt{C_2}} \leq \frac{1}{\sqrt{5 d C_1}} \,.
\]
Now this completes the inductive step since the above now implies
\[
\sum_{i = 3}^{k+1} (\alpha_i - v \cdot b_i)^2 \leq \frac{k - 1}{5 d C_1} \,.
\]
When $k = d$, the inductive hypothesis combined with the assumption that $\alpha_1^2 + \alpha_2^2 \geq \frac{1}{C_1^{0.6}}$ now implies that we return some $u$ with
\[
\norm{u - \frac{v}{\norm{v}}} \leq \frac{1}{C_1^{0.2}}
\] 
as desired.  
\end{proof}

Next, we show that if $b_1, b_2$ are separating (as in \Cref{def:2d-separated}) then we can bound the number of recursive calls in Algorithm~\ref{alg:search}.

\begin{claim}\label{claim:bounded-recursion}
If $b_1, b_2$ are $\gamma/(10n)^3$-separating, then if we start Algorithm~\ref{alg:search} with any $\alpha_1, \alpha_2$ with $\alpha_1^2 + \alpha_2^2 \geq \frac{1}{C_1^{0.6}}$, the algorithm will only recurse on at most one possibility for $\alpha_k$ for each $k \geq 3$. 
\end{claim}
\begin{proof}
By \Cref{lem:where-zero} and the definition of the threshold $\tau = 10 \eps (\pi\ell^2)^{d/2}$, if we fix $\alpha_1, \dots , \alpha_k$, then in the execution of Algorithm~\ref{alg:search}, a value $c$ gets added to the set $\calT'$ only if for some $t \in \R, i \in [n]$,
\[
C_2 (t(v_i \cdot b_{k+1}) - c)^2 + C_2  (t(v_i \cdot b_1) - \alpha_1)^2  +  C_2  (t(v_i \cdot b_2) - \alpha_2)^2 \leq 4 d \log(\ell n/\eps) \,.
\] 
Recall that we set $C_1, C_2 \leq \frac{\ell^2}{d}$ and therefore when the above doesn't hold, the upper bound obtained in \Cref{lem:where-zero} for the choice of $v,A$ in Algorithm~\ref{alg:search} is much smaller than $\tau$.

By the setting of $C_1, C_2$ sufficiently large and the assumption that $b_1, b_2$ are $\theta$-separating for $\theta = \gamma/(10n)^3$, all of the $c$ that can satisfy the above must actually correspond to the same $i \in [n]$.  This also implies that the range of possible $t$ is at most 
\[
|t_{\max} - t_{\min}| \leq \frac{4\sqrt{d \log (\ell n/\eps)}}{\sqrt{C_2}} \cdot \frac{\sqrt{d}}{\theta} = \frac{4d(10n)^3\sqrt{\log (\ell n/\eps)}}{\gamma \sqrt{C_2}} \,.
\]
Since $C_1 = C_2^{0.9}$ and $C_2$ is a sufficiently large polynomial, this implies that all elements $c$ that get added to $\calT'$ must be contained in an interval of width at most $1/\sqrt{10dC_1}$.  Thus, actually the set $\calS$ that gets constructed has size at most $1$.  Thus, the algorithm will recurse on at most one value of $\alpha_{k+1}$ at each step, as desired.
\end{proof}

Finally, we can analyze the full algorithm for recovering the directions $v_i$.  We show that we can recover all of the directions for which $\sigma_i$ is nondegenerate and that also there are no extraneous directions recovered.  The procedure works by first randomly choosing an orthnormal basis $b_1, \dots , b_d$ (note $b_1, b_2$ are separating with high probability by \Cref{claim:anticoncentration}).  We then grid search over the first two coordinates $\alpha_1, \alpha_2$ and run Algorithm~\ref{alg:search} initialized with each possibility.  The analysis uses \Cref{claim:bounded-recursion} to bound the runtime and query complexity and combines \Cref{lem:where-nonzero} with \Cref{claim:find-good-point} to argue that all of the desired directions are found.

\begin{lemma}[Finding Directions]\label{lem:find-directions}
Assume that $f(x) = \sigma_1(v_1^\top x) + \dots + \sigma_n(v_n^\top x)$ is a sum of $n$ features satisfying \Cref{assume:nondegen} and \Cref{assume:bounded}.  Then with parameters as in \eqref{eq:hyperparams} and assuming $d \geq 3, R > 1$ and 
\[
\eps < \frac{1}{\poly(d,n,R,L, \frac{1}{\gamma}, \frac{1}{\Delta})} \,,
\]
there is an algorithm that uses $\poly(d,n,R,L, \frac{1}{\gamma}, \frac{1}{\Delta})$ runtime and queries to a Fourier mass oracle $\calI_{\tau, f^{(\ell)}}$ with $\tau = 10\eps(\pi \ell^2)^{d/2}$, and with probability at least $0.9$ returns a set of unit vectors $ \{u_1, \dots , u_s\}$ with the following properties:
\begin{itemize}
\item For each $j \in [s]$, there is some $i \in [n]$ such that $\min(\norm{u_j - v_i}, \norm{u_j + v_i}) \leq \Delta$.
\item For each function $\sigma_i( \cdot )$ that is $(R, \Delta)-nondegenerate$, there is some $j \in [s]$ with \\ $\min(\norm{u_j - v_i}, \norm{u_j + v_i}) \leq \Delta$.
\item For each $j \neq j'$ the sine of the angle between $u_j$ and $u_{j'}$ is at least $\gamma/2$.
\end{itemize}
\end{lemma}
\begin{proof}
We first randomly choose an orthonormal basis $b_1, \dots , b_d$.  By \Cref{claim:anticoncentration}, with probability at least $1 - 1/(10n)$ over this choice, $b_1, b_2$ are $\gamma/(10n)^3$-separating.  For the remainder of this proof, we condition on this event.

Now, we set parameters $\ell, C_1, C_2$ as in \eqref{eq:hyperparams} and grid over all $\alpha_1, \alpha_2$ with $\frac{1}{C_1^{0.6}} \leq \alpha_1^2 + \alpha_2^2 \leq \ell^4$ with grid size $\frac{1}{10\sqrt{C_2}}$.  Note that the number of such grid points is at most $\poly(d,n,R,L, \frac{1}{\gamma}, \frac{1}{\Delta})$.

For each such grid point, we run Algorithm~\ref{alg:search}.  By \Cref{claim:bounded-recursion}, if $b_1, b_2$ are $\gamma/(10n)^3$-separating, then for each such grid point, Algorithm~\ref{alg:search} recurses on at most one value of $\alpha_k$ for each $k \geq 3$.  Also it is clear that each iteration of Algorithm~\ref{alg:search} can be implemented in $\poly(d,n,R,L, \frac{1}{\gamma}, \frac{1}{\Delta})$ time and oracle queries.  This gives the desired time and query complexity bounds.

Next, we argue about the set of points that we actually find.  By \Cref{lem:where-nonzero}, for each $i \in [n]$ such that $\sigma_i$ is $(R, \Delta)$-nondegenerate, there is some $\beta \in [a,B] \cup [-B,-a]$ (where $a,B$ are as defined in \Cref{lem:where-nonzero}) such that the point $v = \beta v_i$ is heavy.  This is because the estimate in \Cref{lem:where-nonzero} (for $\alpha = 4C_2 = 4\ell^2/d$) can be lower bounded as 
\[
(\pi \ell^2)^{d/2} \left( \frac{\Delta^2}{10^3 \ell^{10}}  - 4n\ell e^{-\ell^{0.2}} \right) > 10^3\eps (\pi \ell^2)^{d/2}
\]
where we used the assumption on $\ell$ being sufficiently large and  $\eps$ being sufficiently small compared to $\ell$. We also immediately have $\norm{v} \leq \ell^2/2$.

Now, since $b_1, b_2$ are $\gamma/(10n)^3$-separating and using the lower bound $|\beta| \geq a$ where $a = \frac{\Delta}{8}\sqrt{\frac{1}{R\ell}}$ (as defined in \Cref{lem:where-nonzero}), there must be some choice of $\alpha_1, \alpha_2$ in our grid with $|v \cdot b_1 - \alpha_1|, |v \cdot b_2 - \alpha_2| \leq \frac{1}{\sqrt{10C_2}} $ because  

\[
(v \cdot b_1)^2 + (v \cdot b_2)^2  = \beta^2 (v_i \cdot b_1)^2 + \beta^2 (v_i \cdot b_2)^2 \geq \frac{\gamma^2\Delta^2}{(20n)^6 d R\ell} > \frac{2}{C_1^{0.6}} \,.
\]
Thus, by \Cref{claim:find-good-point}, when we run Algorithm~\ref{alg:search} starting from this $\alpha_1, \alpha_2$, we will find some point $u$ with
\[
\min(\norm{u - v_i}, \norm{u + v_i}) \leq \frac{1}{C_1^{0.2}}
\]
since by definition $v/\norm{v} = \pm  v_i$.  Note that the above argument holds for any $i \in [n]$ such that $\sigma_i$ is $(R,\Delta)$-nondegenerate and thus for each such $i$, we find some point $u$ with the above property.

Next, we also argue that we do not find any extraneous points that don't correspond to some direction $v_i$.  Since $d \geq 3$, Algorithm~\ref{alg:search} can recurse on $\alpha_1, \dots , \alpha_{k+1}$ for $k \geq 2$ only if 
\[
I^\star_{f^{(\ell)}}(v_0, A_0) \geq 4\tau 
\]
where 
\[
v_0 = \alpha_1 b_1 + \dots + \alpha_{k+1} b_{k+1} \; , \; A_0 = C_2 b_1b_1^\top + C_2b_2b_2^\top + C_1 b_3b_3^\top + \dots + C_1 b_{k}b_{k}^\top + C_2 b_{k+1} b_{k+1}^\top \,.
\]
By \Cref{lem:where-zero} and the way we set parameters $\ell, C_1, C_2$ in \eqref{eq:hyperparams}, this can only happen if for some $i \in [n]$ and $t \in \R$
\[
\sum_{j = 1}^{k+1} (t(v_i \cdot b_j) - \alpha_j)^2 \leq \frac{1}{C_1^{0.8}} \,.
\]
Aso recall that $\alpha_1^2 + \alpha_2^2 \geq \frac{1}{C_1^{0.6}}$.  Thus, any point that Algorithm~\ref{alg:search} actually returns must satisfy the above for $k +1 = d$ and this therefore implies that any returned point $u$ satisfies 
\[
\min(\norm{u - v_i}, \norm{u + v_i}) \leq \frac{1}{C_1^{0.1}}
\]
for some $i \in [n]$.  Thus we've shown so far that we can guarantee the first two of the desired conditions.

Finally, we argue that we can post-process to ensure  separation.  Among all of the returned points, we greedily construct a maximal set of points such that all pairs have the sine of the angle between them being at least $\gamma/2$.  Since the $v_i$ are $\gamma$-separated in angle, this still ensures that if some returned point $u$ is $\frac{1}{C_1^{0.1}}$-close to $\pm v_i$ for some $i$, then after post-processing, there must still be some point remaining that is $\frac{1}{C_1^{0.1}}$-close to $\pm v_i$.  Since $\frac{1}{C_1^{0.1}} < \Delta$ by the way we set $C_1$, we have now verified all three conditions and this completes the proof.
\end{proof}

\section{Function Recovery}\label{sec:function-recovery}

Once we have recovered the directions that are close to the $v_i$, we now show how to recover the actual functions $\sigma_i$.  An important subroutine for this step is using queries to estimate the value of $\wh{f^{(\ell)}}(y)$ at any specified point $y$. 

\subsection{Estimating Fourier Value}

 The subroutine for estimating $\wh{f^{(\ell)}}(y)$ is similar to Algorithm~\ref{alg:estimate-Fourier}, but much simpler in terms of the sampling procedure, and is described below.

\begin{algorithm2e}[ht!]\label{alg:estimate-value}
\caption{Estimating Fourier Value}
\DontPrintSemicolon
\KwIn{Query access to $f_{\sim}:\R^d\to\R$}
\KwIn{$\ell>0$, point $y \in \R^d$, sample budget $m$}

\Fn{$\textsf{EstVal}_{f_{\sim}, \ell,  m}(y)$}{
\For{$j=1$ \KwTo $m$}{
  Draw $x \sim  N_{0, \ell^2 I_d}$\;
  Set $c_j \leftarrow f_{\sim}(x)e^{-\ii y^\top x}$\;
}
\KwRet $V \leftarrow \ell^d \cdot \frac{c_1 + \dots + c_m}{m}$\;
}
\end{algorithm2e}

It is straight-forward to verify that Algorithm~\ref{alg:estimate-value} gives an unbiased estimate for $\wh{f^{(\ell)}}(y)$ and that the estimator concentrates for $m$ sufficiently large.

\begin{claim}\label{claim:val-concentration}
Let $f: \R^d \rightarrow \R$ be a function with $|f(x)| \leq 1$ and fix $\ell > 0$.  Let $f^{(\ell)}$ be as defined in \Cref{def:Gaussian-smoothing}.  Given query access to $f_{\sim}$ with $\|f-f_{\sim}\|_\infty\le \eps$ and any $y \in \R^d$ and sample budget $m$, compute $V = \textsf{EstVal}_{f_{\sim}, \ell, m }(y)$ as defined in Algorithm~\ref{alg:estimate-value}. Then for any $\delta\in(0,1)$, with probability at least $1-\delta$,
\[
\left|\, V - \wh{f^{(\ell)}}(y) \right|
\;\le\; 4 \ell^d \!\left(\eps + \sqrt{\frac{2\log(8/\delta)}{m}}\right).
\]
\end{claim}
\begin{proof}
By definition,
\[
\wh{f^{(\ell)}}(y) = \frac{1}{(2\pi)^{d/2}}\int_{\R^d}  f(x) e^{-\norm{x}^2/(2\ell^2)}  e^{-\ii y^\top x} dx \,.
\]    
From this, it is immediate that if $c_j$ in Algorithm~\ref{alg:estimate-value} were computed using query access to $f$, then its expectation would be $\wh{f^{(\ell)}}(y)/\ell^d$.  Thus, by the assumption about $f_{\sim}$, we have
\[
| \ell^d \E[c_j]  - \wh{f^{(\ell)}}(y) | \leq \eps \ell^d \,.
\]
Now we apply Hoeffding's inequality and since each $c_j$ has $|c_j| \leq 1 + \eps$ this gives the desired concentration.
\end{proof}

In light of Claim~\ref{claim:val-concentration}, we define the following Fourier value oracle, and the rest of the analysis in this section will be in terms of the number of calls to this Fourier value oracle for the function $f^{(\ell)}$. We will put everything together to bound the number of actual queries to $f_{\sim}$ in \Cref{sec:putting-together}. 

\begin{definition}[Fourier Value Oracle]
A Fourier Value Oracle for an underlying function $g$ and accuracy $\tau$ on a query $y \in \R^d$ returns a value $V_{\tau, g}(y)$ such that 
\[
|\wh{g}(y) -  V_{\tau, g}(y)| \leq \tau \,.
\]
\end{definition}

Now our next goal will be to show that for an unknown sum of features $f(x) = \sigma_1(v_1^\top x) + \dots +  \sigma_n(v_n^\top x)$ satisfying \Cref{assume:nondegen} and \Cref{assume:bounded}, if we are given a direction $u$ that is sufficiently close to $v_i$ for some $i \in [n]$, then we can recover the corresponding function $\sigma_i$.  First, we begin by setting parameters  We assume we are given parameters $d,n, R, L, \gamma, \eps$ and $\eps \leq \frac{1}{\poly(d,n,R,L, \frac{1}{\gamma})}$ for some sufficiently large polynomial.  We then set: 
\begin{equation}\label{eq:params2}
\ell = \poly\left(d,n,R,L, \frac{1}{\gamma}\right) \;, \;  \Delta = \frac{1}{\poly(\ell)} \;, \;  \tau = 10 \eps \ell^{d}
\end{equation}
and we assume that $\eps$ is sufficiently small that $\eps \ll 1/\poly(1/\Delta)$.  Note that this differs from the setting in \eqref{eq:hyperparams} because now $\ell \ll 1/\Delta$.  In our full algorithm, we will set two smoothing scales $\ell_1, \ell_2$ with $\ell_2 \ll 1/\Delta \ll \ell_1$ and we will run the first part, described in \Cref{sec:freq-finding} with $\ell = \ell_1$ and the second part, described here with $\ell = \ell_2$.

We first prove the following claim showing that querying $\wh{f^{(\ell)}}$ allows us to get good point estimates for $\wh{\sigma_i^{(\ell)}}$ if we are given a direction $u$ that is close to $v_i$.  We will then use \Cref{claim:truncated-inversion} to reconstruct the function $\sigma_i$ by querying at a discrete grid of points to approximate the integral.

\begin{claim}\label{claim:fourier-point-closeness}
Let $f(x) = \sigma_1(v_1^\top x) + \dots +  \sigma_n(v_n^\top x)$ be a sum of features satisfying \Cref{assume:nondegen} and \Cref{assume:bounded}. Assume that we are given a unit vector $u$ such that $\norm{u - v_i} \leq \Delta$ (with parameters set in \eqref{eq:params2}).  Then for any $t$ with $\frac{1}{\ell^{0.9}} \leq  |t| \leq \frac{1}{\Delta^{0.1}}$, 
\[
|\wh{f^{(\ell)}}(t u) - \ell^{d-1} \wh{\sigma_i^{(\ell)}}(t)| \leq 2\Delta^{0.8} \ell^d \,.
\]
\end{claim}
\begin{proof}
First, by \Cref{claim:fourier-lipchitz-2} and the setting of parameters in \eqref{eq:params2},
\[
|\wh{f^{(\ell)}}(t u) - \wh{f^{(\ell)}}(t v_i)| \leq n \ell^{d+1} |t| \norm{u - v_i} \leq \Delta^{0.8} \ell^d \,.
\]
Next, by \Cref{claim:fourier-formula}, we have the formula
\[
\wh{f^{(\ell)}}(t v_i) = \sum_{j = 1}^n  \wh{\sigma_j^{(\ell)}}(t v_j^\top v_i)\cdot \ell^{d-1} e^{-\ell^2\norm{tv_i - (tv_j^\top v_i)v_j}^2/2}
\]
but for all $j \neq i$, 
\[
\frac{\ell^2\norm{tv_i - (tv_j^\top v_i)v_j}^2}{2} \geq \frac{\ell^2 t^2 \gamma^2}{2} \geq \frac{\ell^{0.2}\gamma^2}{2} \geq \ell^{0.1} \,.
\]
Thus, by the way we set $\ell$, we can ensure the sum of all of the terms for $j \neq i$ is at most $\Delta^{0.8} \ell^d$.  The term for $j = i$ is exactly $ \ell^{d-1} \wh{\sigma_i^{(\ell)}}(t)$ so putting these together, we get
\[
|\wh{f^{(\ell)}}(t u) - \ell^{d-1} \wh{\sigma_i^{(\ell)}}(t)| \leq 2\Delta^{0.8} \ell^d 
\]
as desired.
\end{proof}

\noindent We can now prove the main lemma for this section. 

\begin{lemma}[Given a Direction, Recover the Function]  \label{lem:recover-function}
Let $f : \R^d \to \R$ be a sum  of features $f = \sigma_1(v_1^\top x_1) + \dots +  \sigma_n(v_n^\top x)$ satisfying \Cref{assume:nondegen} and \Cref{assume:bounded}.  Suppose we are given a unit vector $u$ such that there exists some $i \in [n]$ with $\norm{u - v_i} \leq \Delta$.  Then with parameters set as in \eqref{eq:params2}, there is an algorithm that takes $\poly(1/\Delta)$  queries to a Fourier value oracle $V_{\tau, f^{(\ell)}}$ and runtime and outputs a function $\wt{\sigma} : \R \to \R$ with
\[
\max_{x \in \R^d, \norm{x} \leq R} |\wt{\sigma}(u^\top x) - \sigma_i(v_i^\top x)| \leq \frac{5}{\ell^{0.7}} \,.
\]
\end{lemma}
\begin{proof}
Let $\calS$ be the set of all $t$ that are integer multiples of $\Delta$ and satisfy $\frac{1}{\ell^{0.9}} \leq |t| \leq \frac{1}{\Delta^{0.1}}$. For each $t \in \calS$, we query $V_{\tau, f^{(\ell)}}(tu)$.  We then define the function 
\[
\wt{\sigma}(z) = C + e^{\frac{z^2}{2\ell^2}} \frac{\Delta}{\sqrt{2\pi}}\sum_{t \in \calS} e^{\ii tz} \frac{V_{\tau, f^{(\ell)}}(tu)}{\ell^{d-1}} 
\]  
where $C$ is a constant chosen so that $\wt{\sigma(0)} = 0$.  Define the set $\calT$ to include all $t$ where $t$ is an integer multiple of $\Delta$ with $|t| \leq \frac{1}{\ell^{0.9}}$.  Define 
\[
\begin{split}
\phi(z) \defeq e^{\frac{z^2}{2\ell^2}} \frac{\Delta}{\sqrt{2\pi}}\sum_{t \in \calS \cup \calT} e^{\ii tz} \frac{V_{\tau, f^{(\ell)}}(tu)}{\ell^{d-1}}  \\ 
\phi_0(z) \defeq e^{\frac{z^2}{2\ell^2}} \frac{\Delta}{\sqrt{2\pi}}\sum_{t \in \calS \cup \calT} e^{\ii tz} a_i \wh{\sigma^{(\ell)}}(t) 
\end{split}
\]
By the guarantees of the value oracle and \Cref{claim:fourier-point-closeness}, we have for all $z$ with $|z| \leq \ell$,
\[
|\phi(z) - \phi_0(z)| \leq \Delta \cdot \frac{2}{\Delta^{1.1}} \cdot 2(\Delta^{0.8}  + 10\eps) \ell \leq \Delta^{0.6} \,.
\]
Also by \Cref{claim:truncated-inversion} and \Cref{claim:lipschitz-fourier} (which bounds the error from discretizing the integral), we have for all $z$ with $|z| \leq R$
\[
| \sigma_i(z) - \phi_0(z)| \leq \Delta^{0.04} + \frac{2}{\Delta^{0.1}} \cdot 2\Delta \ell^2  \leq \Delta^{0.03} \,.
\]
Finally, note that we must have $\norm{\wh{f^{(\ell)}}}_{\infty} \leq n \ell^d$ and thus the oracle values $V_{\tau, f^{(\ell)}}$ must be bounded by $2n\ell^d$.  Define 
\[
\rho(z) \defeq e^{\frac{z^2}{2\ell^2}} \frac{\Delta}{\sqrt{2\pi}}\sum_{t \in \calT}  e^{\ii tz} \frac{V_{\tau, f^{(\ell)}}(tu)}{\ell^{d-1}} \,.
\]  
Then for any $z$ with $|z| \leq R$, because all of the $t$ in the sum above have $|t| \leq \frac{1}{\ell^{0.9}}$, we have
\[
|\rho(z) - \rho(0)| \leq \frac{2}{\ell^{0.9}} \cdot (2\ell n)  \cdot \frac{5R}{\ell^{0.9}} \leq \frac{1}{\ell^{0.7}} \,.
\]
Thus, since by assumption $\sigma_i(0) = 0$ and we chose the shift $C$ so that $\wt{\sigma}(0) = 0$, combining everything we've shown so far implies that actually for all $z$ with $|z| \leq R$, 
\[
|\wt{\sigma}(z) -  \sigma_i(z)| \leq \Delta^{0.6} + \Delta^{0.03} + \frac{3}{\ell^{0.7}} \leq \frac{4}{\ell^{0.7}} \,.
\]
Finally, to prove the high-dimensional statement in the lemma, note that for any $x \in \R^d$ with $\norm{x} \leq R$, by \Cref{claim:lipschitz-fourier},
\[
|\sigma_i(v_i^\top x) -  \sigma_i(u^\top x)| \leq R \Delta \ell^2  \leq \frac{1}{\ell^{0.7}} \,.
\]
Thus, for all $x$ with $\norm{x} \leq R$, 
\[
|\wt{\sigma}(u^\top x) - \sigma_i(v_i^\top x)| \leq |\wt{\sigma}(u^\top x) - \sigma_i(u^\top x)| + |\sigma_i(v_i^\top x) - \sigma_i(u^\top x)| \leq \frac{5}{\ell^{0.7}}
\]
as desired.
\end{proof}

\section{Putting Everything Together}\label{sec:putting-together}

We are now ready to put everything together and prove our main learning results.  We first prove \Cref{thm:weaker} and then show a simple reduction to remove the boundedness assumption (\Cref{assume:bounded})  to get a more general result in \Cref{thm:main-learning}.

\subsection{Proof of \Cref{thm:weaker}}\label{sec:weaker-proof}

With all of the components that we have so far, the remainder of the proof of \Cref{thm:weaker} proceeds by directly combining \Cref{lem:find-directions} and \Cref{lem:recover-function} and our concentration bounds for implementing the Fourier mass and Fourier value oracles in \Cref{coro:concentration} and \Cref{claim:val-concentration}.

\begin{proof}[Proof of \Cref{thm:weaker}]
Given the parameters $d,n,L,R, \gamma, \eps', \delta$, we can set 
\[
\ell_2 = \poly\left(d,n,L,R, \frac{1}{\gamma}, \frac{1}{\eps'}\right) \;, \; \Delta = \frac{1}{\poly(\ell_2)} \; , \; \ell_1 = \poly\left( \frac{1}{\Delta}\right)
\]
and assume $\eps < 1/\poly(\ell_1)$ all for some sufficiently large polynomials. We apply \Cref{lem:find-directions} with $\ell \leftarrow \ell_1$. By \Cref{coro:concentration}, we can simulate all of the Fourier mass oracle queries with probability $1 - 0.1\delta^2$ using $\poly(\log(1/\delta), 1/\eps)$ actual queries to the function $f_{\sim}$.  Also, we can run the algorithm $O(\log(1/\delta))$ many times independently and post-process the points by majority voting to reduce the failure probability to $0.1\delta$.  We now have a set of unit vectors $\{u_1, \dots , u_s \}$ for some $s \leq n$.  The conditions in \Cref{lem:find-directions} imply that each $u_i$ must be close to exactly one hidden direction $v_j$.  By permuting the indices and possibly negating some of the directions $v_i$, WLOG we may assume that 
\begin{itemize}
\item For each $i \in [s]$, $\norm{u_i - v_i} \leq \Delta$
\item For $i \geq s+1$, the function $\sigma_i(\cdot)$ is not $(R,\Delta)$-nondegenerate
\end{itemize} 
We now apply \Cref{lem:recover-function} on each of $u_1, \dots , u_s$ with $\ell \leftarrow \ell_2$ to obtain functions $\wt{\sigma_1}, \dots , \wt{\sigma_s}$ and then we output our estimate
\[
\wt{f}(x) \defeq \wt{\sigma_1}(u_1^\top x) + \dots + \wt{\sigma_s}(u_s^\top x) \,.
\]
\Cref{claim:val-concentration} implies that we can simulate all of the Fourier value oracle queries with probability $1 - 0.1\delta^2$ using $\poly(\log(1/\delta), 1/\eps)$ actual queries to the function $f_{\sim}$.  Now we bound the error of our estimate. Since $\sigma_i(\cdot)$ is not $(R,\Delta)$-nondegenerate for $i \geq s+1$ we get that for all $x$ with $\norm{x} \leq R$,
\[
|\sigma_{s+1}(v_{s+1}^\top x) + \dots +  \sigma_n(v_n^\top x) |\leq n\Delta
\]
since we assumed that $\sigma_i(0) = 0$ for all $i$.  Thus, the guarantees of \Cref{lem:recover-function} and the way we set the parameter $\ell_2$ give that for all $x$ with $\norm{x} \leq R$,
\[
|\wt{f}(x) - (\sigma_1(v_1^\top x) + \dots + \sigma_n(v_n^\top x))| \leq \eps
\] 
as desired.  The overall failure probability is at most $\delta$ and the total number of queries is polynomial in all relevant parameters so this completes the proof.
\end{proof}

\subsection{Removing Boundedness}\label{sec:remove-bounded}
We now prove \Cref{thm:main-learning} by showing how to remove the assumption that the $\sigma_i$ are bounded (\Cref{assume:bounded}).  We will do this by reducing to the bounded case.  The idea is to first convolve $f$ with a small Gaussian to ensure smoothness. We then apply \Cref{thm:weaker} on the derivatives of $f$ \---- we show that we can simulate query access to the derivatives.  Then since the reconstructed functions are explicitly integrable, we simply integrate to reconstruct $f$.

The proof of \Cref{coro:identifiability}, about identifying all of the directions $v_i$ for which $\sigma_i(\cdot)$ is nonlinear, will also be immediate from the same reduction approach.

\begin{proof}[Proof of \Cref{thm:main-learning}]
Let
\[
\eta = \frac{1}{\poly(d,L,R,n, \frac{1}{\gamma}, \frac{1}{\eps'})}
\]
and define
\[
g(x) =  \int_{\R^d} f(x + z) N_{0, \eta^2 I_d}(z) dz \,.
\]
Now we have
\[
\nabla g(x) =  \int_{\R^d} \nabla f(x + z)   N_{0, \eta^2 I_d}(z) dz  = \sum_{i = 1}^n a_i \int_{\R^d} \nabla  f_i(x + z)  N_{0, \eta^2 I_d}(z) 
\]
where $f_i(x) = \sigma_i(v_i^\top x)$.  Now for any unit vector $u$, we can write
\[
\langle u, \nabla g(x) \rangle = \sum_{i = 1}^n  \int_{\R^d} \langle u, v_i \rangle  \sigma_i'(v_i^\top (x + z))  N_{0, \eta^2 I_d}(z) \,.
\]
First consider fixing $u$ and defining 
\[
\rho_i(t) = \int_{-\infty}^{\infty} \langle u, v_i \rangle  \sigma_i'(t + z)  N_{0, \eta^2}(z) dz \,.
\]
Then we have
\[
\langle u, \nabla g(x) \rangle =  \rho_1 (v_1^\top x) + \dots + \rho_n(v_n^\top x) \,.
\]
The definition of $\rho_i$ immediately gives that $|\rho_i(t)| \leq L$ for all $t \in \R$.  Next, we bound the derivative of $\rho_i$. We can write
\[
\rho_i'(t) = \langle u, v_i \rangle \frac{d}{dt}\int_{-\infty}^{\infty} \sigma_i'(z)N_{0, \eta^2}(t - z) dz = \langle u, v_i \rangle \int_{-\infty}^{\infty}  \sigma_i'(z) \cdot  \frac{t - z}{\eta^2} \cdot N_{0, \eta^2}(t - z) dz \,.
\]
Thus we get that for all $t$, $|\rho_i'(t)| \leq L/\eta$.  Now define $h(x) \defeq \langle u, \nabla g(x) - \nabla g(0) \rangle$.  We have shown that, after rescaling by $\eta/L$, the function $h(x)$ satisfies both \Cref{assume:nondegen} and \Cref{assume:bounded}.  

Now we need show how we can simulate query access to the function $\langle u, \nabla g(x) \rangle$.  Note that from the bounds above, $h(x)$ is $nL/\eta$-Lipschitz.  For any $\alpha$, we can write
\[
\frac{g(x + \alpha u) - g(x)}{\alpha } = \int_{0}^1 \langle \nabla g(x + \alpha t u), u \rangle dt
\]  
and thus 
\[
\left \lvert \frac{g(x + \alpha u) - g(x)}{\alpha } - \langle u, \nabla g(x) \rangle\right \rvert \leq \frac{nL \alpha}{\eta}   \,.
\]
Now with query access to $f_{\sim}$ with $\norm{f - f_{\sim}}_{\infty} \leq \eps$, we can estimate $g$ to $2\eps$ accuracy using $\poly(1/\eps)$ queries simply by sampling.  By the above inequality, this lets us estimate $\langle u, \nabla g(x) \rangle$ to accuracy  $\frac{4\eps}{\alpha} + \frac{n L \alpha}{\eta}$.  Thus, we can set $\alpha = \eps^{0.5}$ and get an $\eps^{0.4}$-accurate oracle for $\langle u, \nabla g(x) \rangle$ as long as $\eps$ is sufficiently small in terms of $\eta$.  Thus, we can now apply \Cref{thm:weaker} (redefining parameters so that $\eps' \leftarrow \eta$) to recover $h(x)$ to accuracy $\eta$ on the domain $\norm{x} \leq R$.  Since we can also just estimate the constant $\langle u, \nabla g(0) \rangle$ and add it back in, we now have an $\eta$-accurate approximation to the function $\langle u, \nabla g(x) \rangle$ as a sum of at most $n$ ridge functions.

Recall that originally we fixed a choice of $u$, but we can actually apply the above for a collection of different $u$, say the set of standard basis vectors $e_1, \dots , e_d$.  This gives us an estimate that is $\sqrt{d}\eta$ accurate for $\nabla g(x)$ on the domain $\norm{x} \leq R$.  Note that the functions returned by \Cref{thm:weaker} are in an explicit form as a sum of ridge functions and thus we can integrate these estimates to get a sum of ridge functions in the same directions.  Thus, we now have some $\wt{g}$ of the desired form such that $|\wt{g}(x) -g(x)| \leq R\sqrt{d}\eta$ on the domain $\norm{x} \leq R$.  Finally, since $f$ is $nL$-Lipschitz, $\norm{g - f}_{\infty}\leq 2 \sqrt{d} nL \eta$ and thus, since we chose $\eta$ sufficiently small, 
\[
|\wt{g}(x) - f(x)| \leq 3RdnL \eta  \leq \eps'
\]
on the domain $\norm{x} \leq R$ and we are done.

\end{proof}

\begin{proof}[Proof of \Cref{coro:identifiability}]
The proof follows from the same reduction as in \Cref{thm:main-learning} but just applying \Cref{lem:find-directions} instead of \Cref{thm:weaker}.
\end{proof}

\bibliographystyle{plain}
\bibliography{bibliography.bib}

\end{document}